%% file: main.tex
\documentclass[10pt,journal]{IEEEtran}

\title{Stabilizing RNN Gradients through Pre-training}

\author{%
Luca Herranz-Celotti \quad Jean Rouat \\
NECOTIS, Universit\'e de Sherbrooke, Canada\\
\texttt{\{luca.celotti,jean.rouat\}@usherbrooke.ca}
}

\date{February 2022}
\usepackage{style/extra_style}

\begin{document}

\maketitle

\input{sections/abstract}

\input{sections/introduction}

\input{sections/methodology}
\input{sections/results}

\input{sections/conclusions}
\input{sections/acknowledgements}

\bibliographystyle{unsrt} 
\bibliography{references}

\input{sections/appendix}

\end{document}

%% file: sections/abstract.tex
\begin{abstract}
    Numerous theories of learning propose to prevent the gradient from  exponential growth with depth or time, to stabilize and improve training. Typically, these analyses are conducted on feed-forward fully-connected neural networks or simple single-layer recurrent neural networks, given their mathematical tractability. In contrast, this study demonstrates that pre-training the network to local stability can be effective whenever the architectures are too complex for an analytical initialization. Furthermore, we extend known stability theories to encompass a broader family of deep recurrent networks, requiring minimal assumptions on data and parameter distribution, a theory we call the Local Stability Condition (LSC). Our investigation reveals that the classical Glorot, He, and Orthogonal initialization schemes satisfy the LSC when  applied to feed-forward fully-connected neural networks. However, analysing  deep recurrent networks, we identify a new additive source of exponential explosion that emerges from counting gradient paths in a rectangular grid in depth and time. We propose a new approach to mitigate this issue, that consists on giving a weight of a half to the time and depth contributions to the gradient, instead of the classical weight of one. Our empirical results confirm that pre-training both feed-forward and recurrent networks, for differentiable, neuromorphic and state-space models to fulfill the LSC, often results in improved final performance. This study contributes to the field by providing a means to stabilize networks of any complexity. Our approach can be implemented as an additional step before pre-training on large augmented datasets, and as an alternative to finding stable initializations analytically.
\end{abstract}

%% file: sections/introduction.tex
\section{Introduction}
\label{sec:intro}



Despite all the efforts to mitigate the negative effect that gradient explosion has on learning \cite{hochreiter1991untersuchungen, bengio1994learning, hochreiter1997long, hochreiter2001gradient, glorot2010understanding, orthogonal_initialization, he2015delving}, how to properly initialize deep recurrent networks ($d$-RNN) remains an open problem. In fact, the literature has focused on simple architectures that are either deep but shallow in time (FFN for feed-forward networks) \cite{glorot2010understanding,orthogonal_initialization,he2015delving,roberts2022principles, defazio2022scaling}, or shallow in layers and deep in time ($1$-RNN) \cite{bengio1994learning,hochreiter1997long, hochreiter2001gradient, pascanu2013difficulty, arjovsky2016unitary,  kerg2019non}, unintentionally missing the effect that the combined depth in layers and in time have on learning. 
Meanwhile, in the era of Transformer-based architectures \cite{vaswani2017attention, brown2020language}, the RNN hidden state has been introduced back from oblivion to allow the attention networks to observe longer sequences \cite{dai2019transformer, katharopoulos2020transformers, poli2023hyena, peng2023rwkv}, renewing the interest in such a fundamental piece of the learning tool-kit. Similarly, a new line of research shows great promise in classification problems for very long sequences, by using clever RNNs,  known as state-space models, where the Legendre Memory Unit, the S4, S5 and Linear Recurrent Unit, are some examples \cite{hangos2004analysis, LMU, gu2022efficiently,smith2023simplified, lru}.
Moreover, a good understanding of RNN dynamics is impactful in many scientific communities that make use of parametrized dynamical systems. This is the case for example when modeling biological neurons in computational neuroscience, or in neuromorphic computing, where recurrency is necessary to take advantage of highly energy efficient devices \cite{henderson2020towards,blouw2019benchmarking, 9395703, lapique1907recherches, izhikevich2003simple}. 
However, these neuron definitions are often complex, making standard parameter initialization strategies from deep learning sub-optimal, and designing a specific initialization strategy for each, cumbersome. Hence, there is a need for a tool to stabilize a wide variety of network definitions before training, and therefore, for broad theories of $d$-RNN stability to justify such a tool.


To address such concerns, here we propose a pre-training to stability method that is simple to apply to a wide variety of architectures and neuron definitions, including among others differentiable and neuromorphic neurons, such as the LSTM, GRU and ALIF neurons \cite{hochreiter1997long, chung2015gated, lapique1907recherches, lsnn}. Additionally, we extend existing initialization theories, such as Glorot, Orthogonal and He \cite{glorot2010understanding, orthogonal_initialization, he2015delving}, to be able to describe the stability of a wider family of $d$-RNN, and we quantify how the variance of the parameter update depends on time and depth. In contrast with classical descriptions \cite{hochreiter1991untersuchungen, bengio1994learning, kerg2019non}, we find two sources of gradient exponential explosion. One is  multiplicative, often described in the FFN and $1$-RNN literature, the other is additive, and is  consequence of adding all the gradient paths in the time and depth grid. To reduce the variance from exponential to linear growth, we propose one sufficient condition, the Local Stability Condition (LSC), and we propose to weight the depth and time components of the gradient as a half with $d$-RNNs, given that the classical weight of one leads to the additive exponential explosion. 
Finally, we show experimentally that our pre-training to stability to attain the LSC, leads to better final performance for a broad range of $d$-RNNs, for both, differentiable and neuromorphic models \cite{hochreiter1997long, chung2015gated, lapique1907recherches, lsnn}.

Our main contributions are summarized below:

\begin{itemize}[noitemsep,nolistsep]
    \item We propose a pre-training to local stability method, that removes the need for architecture specific initializations, Sec.~\ref{sec:pretrain};
    \item We propose the Local Stability Condition (LSC) as an extension of existing stability theories, to describe the gradient stability of a wide family of multi-layer recurrent networks, with minimal assumptions on parameter and data distribution, Sec.~\ref{sec:lsc};
    \item We show that a new form of exponential explosion arises, due to the addition of gradient paths in the time and depth grid, and we propose a method to suppress it, Sec.~\ref{sec:lsc};
    \item We prove that the Glorot \cite{glorot2010understanding}, He \cite{he2015delving}, and Orthogonal \cite{orthogonal_initialization} initializations are special cases of the LSC, Sec.~\ref{sec:convislsc};
    \item We show on feed-forward and recurrent architectures, on differentiable, neuromorphic and state-space models, that pre-training  to satisfy the LSC can improve final performance experimentally, Sec.~\ref{sec:results}.
\end{itemize}

%% file: sections/methodology.tex
\section{What makes deep RNNs different \\ from shallow RNNs and deep FFNs?}

\subsection{General deep RNN and Notation}

Given that neither Glorot, Orthogonal and He, nor neuron specific 1-RNN theories describe $d$-RNN, we want to generalize those established theories of gradient stability to a wider family of $d$-RNN. Let us consider a stack of general recurrent layers, that will make the mathematical analysis simpler:
\begin{equation}\label{eq:general_sys2}
  \begin{aligned}
    \boldsymbol{h}_{t,l} =& \ g_h( \boldsymbol{h}_{t-1,l}, \boldsymbol{h}_{t-1,l-1},  \boldsymbol{\theta}_l) \\
    \hat{\boldsymbol{o}}_t =& \ g_o(\boldsymbol{h}_{t,L},  \boldsymbol{\theta}_o)
  \end{aligned}
\end{equation}
\noindent where we index time and layer with $t\in\{0, \cdots, T\},l\in \{0,\cdots, L\}$, and the hidden state $\boldsymbol{h}_{t,l}$ depends on the previous time-step and layer in a general form $g_h$. Somewhat in contrast with Machine Learning notation, we use causal time indices, more in line with Dynamical Systems notation, so, terms on the left of the equation are at least one time step ahead than terms on the right if the same variable name appears both sides. Notice that whenever the neuron of interest is defined with several hidden states, such as the LSTM or the ALIF, in our notation $\boldsymbol{h}_{t,l}$ will represent the concatenation of all of them. We represent as $\boldsymbol{h}_{t,0}$,  the task input data fed to the first layer, $l=1$. We perform the analysis for the most general form of $g_h, g_o$, with the only assumption that they are first order differentiable, or augmented with a surrogate gradient wherever they were not \cite{esser2016convolutional, zenke2018superspike, lsnn}. 
We denote vectors as lower case bold $\boldsymbol{a}$, matrices as upper case $A$, and their elements as lower case $a$. The variable $\hat{\boldsymbol{o}}_t$ represents the network output, where the hat means that it is the network approximation of the true task output $\boldsymbol{o}_t$.
If $\boldsymbol{h}_{t,l}\in\mathbb{R}^{n_l}$, we call $\boldsymbol{\theta}_l\in\mathbb{R}^{m_l}$ all the learnable parameters in layer $l$, where the number of parameter elements $m_l$ does not need to coincide with the layer width $n_l$, and similarly for $\boldsymbol{\theta}_o\in\mathbb{R}^{m_o}$. 
We call the matrices $M_k\in\{\partial \boldsymbol{h}_{t,l}/\partial \boldsymbol{h}_{t-1,l}, \partial \boldsymbol{h}_{t,l}/\partial \boldsymbol{h}_{t-1,l-1}\}_{t,l}$, the \textit{transition derivatives}, from one state to the following in time or in depth. We call $\rho(M_k)$  the transition derivative radius \cite{belitskii2013matrix, horn1990}, which is the largest modulus of its eigenvalues and it is often used to describe a matrix magnitude.

\subsection{The Local Stability Condition (LSC)}
\label{sec:lsc}

We say that the network satisfies the Local Stability Condition (LSC) when the expectation of each transition derivative radius is one or a half. We prove in App.~\ref{app:lsc} the following

\begin{restatable}[Local Stability Condition, with radii $\E\rho\in \{0.5, 1\}$]{theorem}{ulscboth}
\label{thm:ulscboth} Be the multi-layer RNN in eq.~\ref{eq:general_sys2}. Setting the radii of every transition derivative $M_k$ to $\E\rho=1$ gives an upper bound to the parameter update variance that increases with time and depth as the binomial coefficient $\frac{1}{T}\binom{T + L +2}{T}$. Instead, setting the radii to $\E\rho=0.5$ gives an upper bound that increases \mbox{linearly as $T$}.
\end{restatable}

First of all, we see that for $\E\rho=1$, the bound to the variance of the parameter update grows with depth $L$, while it does not for $\E\rho=0.5$. Moreover, this result is particularly intriguing because $\E\rho=1$ is known to annihilate the multiplicative sources of exponential explosions in FFN and $1$-RNN \cite{bengio1994learning,hochreiter1997long, hochreiter2001gradient, arjovsky2016unitary, pascanu2013difficulty, glorot2010understanding, he2015delving, orthogonal_initialization}, so, the variance is sub-exponential whenever $L$ or $T$ are fixed and the other tends to infinity. However, we observe that a $d$-RNN brings a new source of exponential explosion that is additive, and comes from an exponentially increasing number of gradient paths when we take $T,L\rightarrow\infty$ simultaneously, as shown in Fig.~\ref{fig:subexp} and proven in Lemma~\ref{thm:binomialgrad} App.~\ref{app:counting}. Instead, setting $\E\rho=0.5$ results in a variance that still increases with time and depth, but linearly, i.e. sub-exponentially, even when $T,L\rightarrow\infty$ together. 
This choice of radius effectively results in a gradient that is locally averaged through time and depth with the most uninformative and entropic prior \cite{jaynes1968prior}. Also notice that if we call \mbox{$\rho_{time}=\rho(\partial \boldsymbol{h}_{t,l}/\partial \boldsymbol{h}_{t-1,l})$} and $\rho_{depth}=\rho(\partial \boldsymbol{h}_{t,l}/\partial \boldsymbol{h}_{t-1,l-1})$, the bound is still valid if $\E\rho = (\E\rho_{time}+\E\rho_{depth})/2=0.5$, and thus we can decide to give a stronger weight to the time or depth component, as long as we diminish the weight to the other so as to meet the overall $0.5$ expectation.

Note that Thm.~\ref{thm:ulscboth} provides an upper bound. The proof relied on using matrix norms that are sub-additive and sub-multiplicative \cite{belitskii2013matrix, horn1990}. However, we prove that for general matrices there is no matrix function that is super-additive and super-multiplicative  (see Thm.~\ref{thm:ssn}), and hence, a lower bound cannot be found in general using a similar methodology. Nevertheless, if we restrict ourselves to positive semi-definite matrices, the determinant is super-additive and multiplicative. Therefore, a lower bound to the variance can be satisfied for that restricted family of matrices (see Thm.~\ref{thm:llsc}). Notice that the positive semi-definite requirement, could be in practice encouraged through a loss term during pre-training. Satisfying only the lower bound would therefore require all the transition derivatives to have determinant of one, whereas satisfying both the upper and lower bounds would require all their eigenvalues to be one. However,  experimentally we found it difficult to stabilize through pre-training the lower bound or both bounds simultaneously, and despite being an interesting direction for future work, we chose to focus on stabilizing the upper bound.


\subsection{Conventional initialization methods are LSC}
\label{sec:convislsc}

It is remarkable that our proposed LSC is not at odds with existing theories.
In fact, we prove in App.~\ref{sec:equivalence} that the three major initialization schemes in the Deep Learning literature, are particular cases of the LSC when applied to FFN:

\begin{restatable}{corollary}{glohe}
\label{thm:glohe}Glorot \cite{glorot2010understanding} and Orthogonal \cite{orthogonal_initialization} initializations for linear FFN, and He \cite{he2015delving} initialization for ReLU FFN, are special cases of the $LSC$ for $\E\rho=1$.
\end{restatable}

\begin{figure*}
    \centering
    \includegraphics[width=\textwidth]{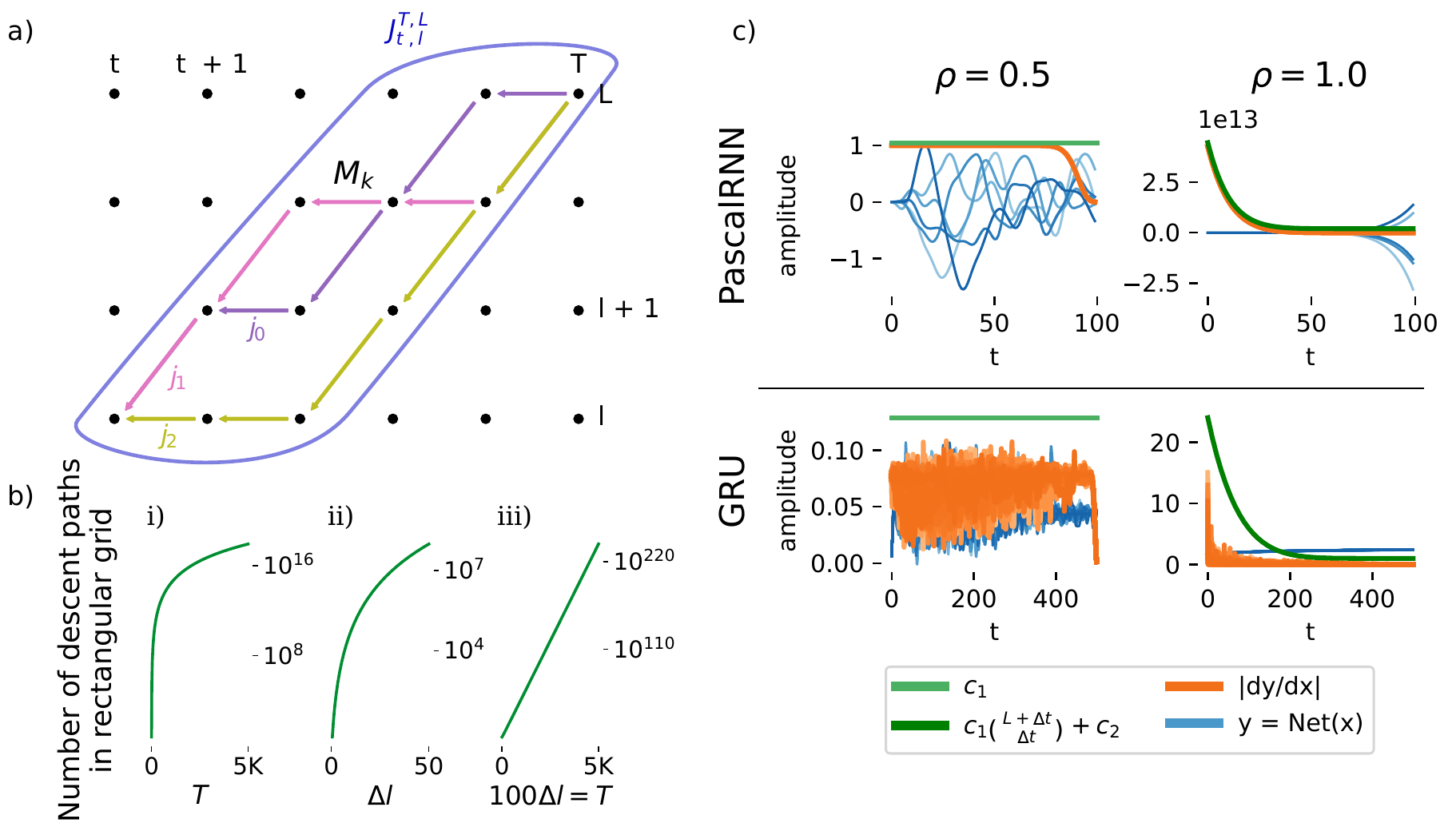}
    \caption{\textbf{Stabilizing to $\rho=1$ results in additive explosion while $\rho=0.5$ does not.} a) In a $d$-RNN gradients need to traverse both the time and the depth dimension when updating the learnable parameters. A transition derivative $M_k$ represents only one arrow in the time and depth grid, and there are several multiplicative chains $j_i$ to be considered, since the parameter update is going to use {\small$J^{T,L}_{t,l}$}, the sum of all the multiplicative chains from $T,L$ down to $t,l$. b) However, the number of paths $j_i$ is described by the binomial coefficient {\small$\binom{\Delta l + \Delta t}{\Delta t}$}, and therefore increases exponentially when time and depth tend to infinity simultaneously, as in iii) and proven in App.~\ref{app:counting}. In fact, an exponential growth looks like a straight line in a semi-log plot, as in iii). Instead, the aforementioned binomial coefficient grows only polynomially when either time or depth are kept fixed, as in i) and ii).  c) We confirm experimentally our theoretical analysis, on a toy network and on the LSC pre-trained GRU: $\rho=1$ reveals an explosion of additive origin (right panels), while $\rho=0.5$ is able to stabilize gradients through time (left panels).
    The upper panels show
    network output (blue), derivative (orange), and our derivative bounds (green), for the toy network that we define as the PascalRNN, $\boldsymbol{h}_{t,l}=\rho\boldsymbol{h}_{t-1,l}+\rho\boldsymbol{h}_{t-1,l-1}$, of depth 10 and gaussian input of mean zero and standard deviation 2, and lower panels show a LSC pre-trained GRU network of depth 7 and the SHD task as input. Both upper bounds to the derivative under $\rho\in\{0.5,1\}$, are part of the proof for Thm.~\ref{thm:ulscboth}, and $c_1, c_2$, defined in Thm.~\ref{thm:ulscboth} proof, are task and network dependent constants, that do not depend on time nor depth. Notice that the growth of the derivative and of the bound is backwards in time since backpropagation accumulates gradients backwards in operations, from $T,L$ to $0,0$. This confirms that standard FFN theories ($\rho=1$) cannot be directly applied to $d$-RNN, since they result in an unexpected additive gradient exponential explosion that is not accounted for.}
    \label{fig:subexp}
\end{figure*}

For the previous Corollary, we needed to prove the following Lemma on random matrices

\begin{restatable}[Expected Spectral Norm]{lemma}{kost}\label{thm:kost}
	Given a $n\times n$ real random matrix, with elements drawn from a radially symmetric point process with mean zero and variance one, its $k$ eigenvalues in expectation have modulus $\E|\lambda_k|=\sqrt{2}\Gamma((k+1)/2)/\Gamma(k/2)$ for any $k$, with vanishing variance as $k$ increases.
\end{restatable}

\noindent where $\Gamma(\cdot)$ is the gamma function. Lemma~\ref{thm:kost} is a version of the strong circular law \cite{tao2010random} for random matrix eigenvalues that is valid  in expectation for any matrix size. 

\subsection{Pre-train to stability}
\label{sec:pretrain}

\noindent\textbf{Gradient Descent Pre-training.}
A limitation of current initialization methods is that every time a new architecture is introduced, a new theoretically justified initialization should be calculated to ensure network stability, and avoid the gradient exponential explosion.
To address this, we suggest a pre-training approach where the network is trained to meet the LSC using the desired task inputs with randomized output classes, as a means of preparing the network for training.
In the preparation phase, we minimize only the mean squared error loss between the radii $\rho$ of all the transition derivatives $M_k$, and a target radius $\rho_t$, at each time step in the task:
\begin{align}    \mathcal{L}_{preparation} = \sum_{\forall M_k} \Big(\rho(M_k)-\rho_t\Big)^2
\end{align}
Similarly, \cite{li2019orthogonal} suggests pre-training the weight matrices to be orthogonal but not the transition derivatives themselves, and \cite{wangenhancing} suggests a pre-training to orthogonality only for the output linear layer. 

\noindent\textbf{Weight Multiplier and Shuffled Pre-training.}
However, gradient descent pre-training results in $M_k$ satisfying the LSC on average, with a wide variance. In order to reduce such variance, after applying the gradient update, we multiply the neuron weights by $\kappa=clip(\rho_t/\rho(M_k))$, which increases the scale of the weights if the radius is smaller than the target, and viceversa. We clip the multiplicative factor between 0.85 and 1.15, to reach the target value gradually and to reduce oscillations. At each pre-training step, we multiply the input matrices of layer $l$ by the layer transition derivative $\kappa$ of layer $l$, and we multiply the recurrent matrices by the time transition derivative $\kappa$ of layer $l$. 
To improve the randomness of the pre-trained model in an attempt to capture the statistics of the task and not to memorize the pre-training set, we shuffle each learnable tensor after applying $\kappa$. We consider pre-training complete, if the following three criteria are met: (i) $|\rho(M_k)-\rho_t|\leq\epsilon$ with $\epsilon=0.02$, (ii) the standard deviation of $\rho(M_k)$ across $k$ in the current pre-training step is below $0.2$, and (iii) a 10 steps exponential moving average of that standard deviation is below $0.2$ also. These thresholds are chosen to make sure that the distance between $\rho\rightarrow0.5$ and $\rho\rightarrow1$ is statistically significant.

\subsection{Neural Networks}

\noindent\textbf{FFNs.} To prove that our pre-training to stability method is applicable to a wide variety of neural networks we apply it first to simple 30 layers FFNs. They are defined as $\boldsymbol{h}_{l} = a( W_{l}\boldsymbol{h}_{l-1} + \boldsymbol{b}_l)$, with $W_l, \boldsymbol{b}_l$ learnable matrices and biases, with three different activations $a$: ReLU, sine and cosine. 
While applications of the sine and cosine activations do exist \cite{vaswani2017attention, sitzmann2020implicit, poli2023hyena}, here our interest is on using unusual activations as a proxy for a real world scenario, where we do not have a theoretically analyzed architecture, and we therefore lack of a so called correct initialization. When we apply the LSC to FFN the target is always $\rho_t=1$.  Notice that we use such a deep network to test the ability of each approach to stabilize it. Being only a simple FFN, it will not achieve state-of-the-art results. Our intention is also to proceed similarly as it was done to introduce He initialization \cite{he2015delving}, where a 30-layers ReLU network initialized with He was compared with one initialized with Glorot. In contrast,  
we show in Fig.~\ref{fig:ffnplot} results for 7 learning rates, 4 seeds, 3 datasets and 3 activations, therefore 252 times more data than in Fig.~3 in \cite{he2015delving}.

\noindent\textbf{RNNs.} Then, we study our LSC preparation method on six different RNNs: four are fully differentiable, and two are non-differentiable but upgraded with a surrogate gradient to be able to train them through back-propagation. The fully differentiable networks are: the LSTM \cite{hochreiter1997long}, the GRU \cite{chung2015gated}, and two fully-connected simple RNN, defined as $\boldsymbol{h}_{t,l} = a(W_{rec,l}\boldsymbol{h}_{t-1,l} + W_{in,l}\boldsymbol{h}_{t-1,l-1} + \boldsymbol{b}_l)$ either with the activation $a$ being the sigmoid $\sigma$, or  $ReLU$, where $W_{rec,l}, W_{in,l}, \boldsymbol{b}_l$ are learnable matrices and biases. As non-differentiable architectures we used two variants of the ALIF neuron \cite{lsnn, lapique1907recherches, yin2021accurate, gerstner2014neuronal}, a simplification of a biologically plausible neuron that is often used in Computational Neuroscience. The exact equations and initializations used are reported in App.~\ref{app:morenets}. We call ALIF$_{+}$ the configuration that is initialized with $\beta$ positive, and ALIF$_{\pm}$ the same initialization with $\beta$ both positive and negative. The variable $\beta$ is also known as \textit{spike frequency adaptation}, and when negative, firing encourages further firing, while if positive firing discourages further firing \cite{sfa_darjan}.

The RNN $\sigma$, RNN $ReLU$ and GRU have only one hidden state, while LSTM and ALIF have two, so we concatenate them into one state $\boldsymbol{h}_{t,l}$ to calculate $M_k$.
The RNN we train have depth $L=2$ and $5$. The layer width is task and neuron model specific, and is chosen to have a comparable amount of parameters across neuron types  (App.~\ref{app:trainingdetails}).
We pre-train the differentiable architectures to satisfy  $\rho_t=1$ and $\rho_t=0.5$. Instead, in the case of non-differentiable architectures we only set $\rho_t=1$, given that it was difficult to converge to $\rho_t=0.5$, not only with two surrogate gradient shapes, the derivative of the fast sigmoid \cite{zenke2018superspike} and a multi-gaussian \cite{yin2021accurate}, but even when learnable dampening and sharpness were used. The default initialization of the recurrent matrix is always Orthogonal, as often suggested \cite{Le2015ASW, li2018independently, henaff2016recurrent}.

\subsection{The Datasets}

We train the FFN on MNIST \cite{mnist},  CIFAR10 and CIFAR100 datasets \cite{Krizhevsky2009}. We describe in the following the datasets used for the RNN.
More details can be found in App.~\ref{app:trainingdetails}.

\noindent\textbf{Spike Latency MNIST (sl-MNIST):} the MNIST digits \citep{mnist} pixels (10 classes) are rescaled between zero and one, presented as a flat vector, and each vector value $x$ is transformed into a spike timing using the transformation $T(x) = \tau_{eff}\log(\frac{x}{x-\vartheta})$ for $x>\vartheta$ and  $T(x) = \infty$ otherwise, with $\vartheta=0.2, \tau_{eff}=50 \text{ms}$ \cite{zenke2021remarkable}. The input data is a sequence of $50$ms, $784$ channels ($28\times 28$), with one spike per row.

\noindent\textbf{Spiking Heidelberg Digits (SHD):} is a dataset developed to benchmark spiking neural networks \cite{cramer2020heidelberg}. It is based on the Heidelberg Digits (HD) audio dataset which comprises 20 classes, of ten spoken digits in English and German, by 12 individuals. These audio signals are encoded into spikes through an artificial model of the inner ear and parts of the ascending auditory pathway.

\noindent\textbf{PennTreeBank (PTB):} is a language modelling task. The PennTreeBank dataset \cite{ptb}, is a large corpus of American English texts. We perform next time-step prediction at the word level. The vocabulary consists of 10K words.

\subsection{Experimental set-up}

We pre-train on the train set, with the AdaBelief optimizer \cite{zhuang2020adabelief}, with a learning rate of $3.14\cdot e^{-3}$ and a weight decay of $1\cdot e^{-4}$ until convergence to the target, for both the FFN and RNN.
We train all our FFN and RNN networks with crossentropy loss and AdaBelief optimizer \cite{zhuang2020adabelief}. AdaBelief hyper-parameters are set to the default, as in \cite{radford2018improving, zenke2021remarkable}. We always use the optimizer as a Lookahead Optimizer \cite{zhang2019lookahead}, with synchronization period of 6 and a slow step size of 0.5. We early stop on validation, with a patience of 10 epochs, and report 4 random seeds mean and standard deviation on the test set. 
On PTB we use perplexity as a metric \cite{jelinek1977perplexity}, the lower the better, while on sl-MNIST and SHD, we use what we call the mode accuracy, the higher the better: the network predicts the target at every timestep, and the chosen class is the one that fired the most for the longest. For more training details see App.~\ref{app:trainingdetails}.

%% file: sections/results.tex
\begin{figure*}    
\includegraphics[width=\textwidth]{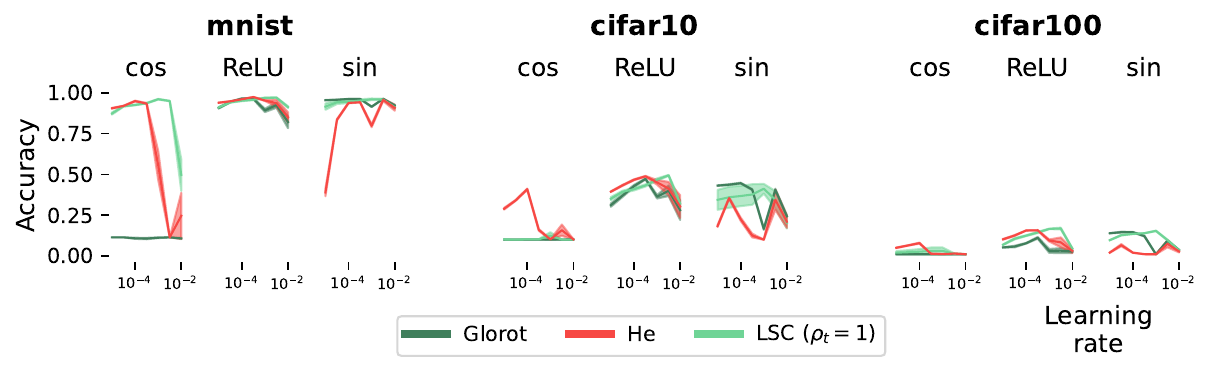}
\caption{\textbf{Bounds stabilization through pre-training enhances FFN learning.} We investigate the effect that pre-training  to achieve our Local Stability Condition (LSC) has on learning, for 30-layer FFNs. Such depth is not necessary to solve MNIST, and our interest is rather in confirming that we are able to stabilize gradients in very deep networks.
Specifically, we stabilize the upper bound and we compare our results to two well-known initialization strategies for FFN, Glorot and He.
Interestingly, even when the theoretically justified He initialization is available, such as for ReLU FFNs, the LSC can match or outperform it, even if it was initialized as Glorot before pre-training. In general, no theoretically justified initialization strategies are available for new architectures or unconventional activations, such as sine and cosine. Therefore, stability pre-training becomes a convenient approach to enhance learning.
Notably, stabilizing to LSC tends to outperforms all other alternatives in most scenarios. }

\label{fig:ffnplot}
\end{figure*}

\section{Results}
\label{sec:results}

\subsection{Pre-training FFN to LSC improves training}

We can see in Fig.~\ref{fig:ffnplot}, the effect on performance of pre-training FFNs to satisfy the LSC,  and we compare it to the classical Glorot and He initializations. We repeat the training with different learning rates, and report test accuracy after early stopping on validation loss. We find that 
pre-training to LSC, outperforms the rest in most scenarios. It is especially interesting that even if on ReLU it tends to match in performance the theoretically equivalent He, it does so at a higher learning rate, with a very similar accuracy/learning rate curve but shifted in the $x$-axis to the right. The reason that there is no exact match, could be consequence of the stochasticity introduced by the mini-batch gradient based pre-training. For the cosine activation, it improves over Glorot, but performs worse than He on CIFAR10 and CIFAR100, which could be consequence of the cosine having a derivative of zero around input zero, which leads to a radius of the transition derivative that has a strong tendency to zero. However, notice that we always initialized the network as Glorot before pre-training to LSC, and the pre-training systematically improved performace with respect to Glorot. This seems to indicate that LSC always improves it's starting initialization.

\subsection{Additive source of exponential explosion in $d$-RNN}

In contrast with FFNs, when we consider deep RNNs, there is a new source of gradient exponential explosion that is additive.
We see in Fig.~\ref{fig:subexp}a) three ways the gradient has to descend through layers and time in a $d$-RNN, from $L,t$ down to $l, t$. Each arrow represents one transition derivative $M_k$. This gradient grid is analogous to the Pascal Triangle grid,  and as such, it hides an additive source of exponential explosion, since the number of paths increases exponentially as we move further away from the upper right corner. The number of shortest paths from the upper right corner to the lower left corner in the grid can be counted with the binomial coefficient $\binom{\Delta t + \Delta l}{\Delta t}$, where $\Delta l = L-l, \Delta t = T-t$. In fact, to prove Thm~\ref{thm:ulscboth} for $\E\rho=1$, we count how many gradient paths descend from $T,L$ down to $0, l$, and mathematically prove in  Lemma~\ref{thm:binomialgrad}  App.~\ref{app:counting} that there are sub-exponentially many when we increase either time or depth leaving the other fixed, and exponentially many when both are increased together. Fig.~\ref{fig:subexp}b) confirms experimentally that the binomial coefficient that counts the number of shortest gradient paths in a grid of sides $\Delta l$ and $T$, grows sub-exponentially i) with $T$ when $\Delta l$ is fixed to $\Delta l=5$ and ii) with $\Delta l$ when $T$ is fixed to $T=5$ and grows exponentially iii) when both are increased together with a fixed ratio $T/\Delta l=100$. In fact, an exponential growth results in a straight line in a semi-log plot, as in the rightmost panel. In Fig.~\ref{fig:subexp}c) we see how a simple $d$-RNN that we call PascalRNN, and GRU illustrate Thm.~\ref{thm:ulscboth} validity. We define as PascalRNN the toy network $\boldsymbol{h}_{t,l}=\rho\boldsymbol{h}_{t-1,l}+\rho\boldsymbol{h}_{t-1,l-1}$, and we show the experimental curves for $L=10$ and $T=100$, with 6 gaussian samples as input of mean zero and standard deviation of 2. We  see that the derivative follows exactly the binomial coefficient for $\rho=1$, and follows the constant $1$ for $\rho=0.5$, also found in the proof of Thm.~\ref{thm:ulscboth}. The shift between the orange and green curves has been introduced manually to ease visualization, but were completely overlapping otherwise. Similarly, we see that the bounds are respected in the panel below, where a $L=7$ stack of GRU neurons was pretrained on the SHD task to satisfy either  $\rho\in\{0.5,1\}$.
Therefore, both networks support the claim that $\E\rho=0.5$ encourages a variance of the gradient that is constant through time, while $\E\rho=1$ does not.

\begin{figure}
    \centering
    \includegraphics[width=.5\textwidth]{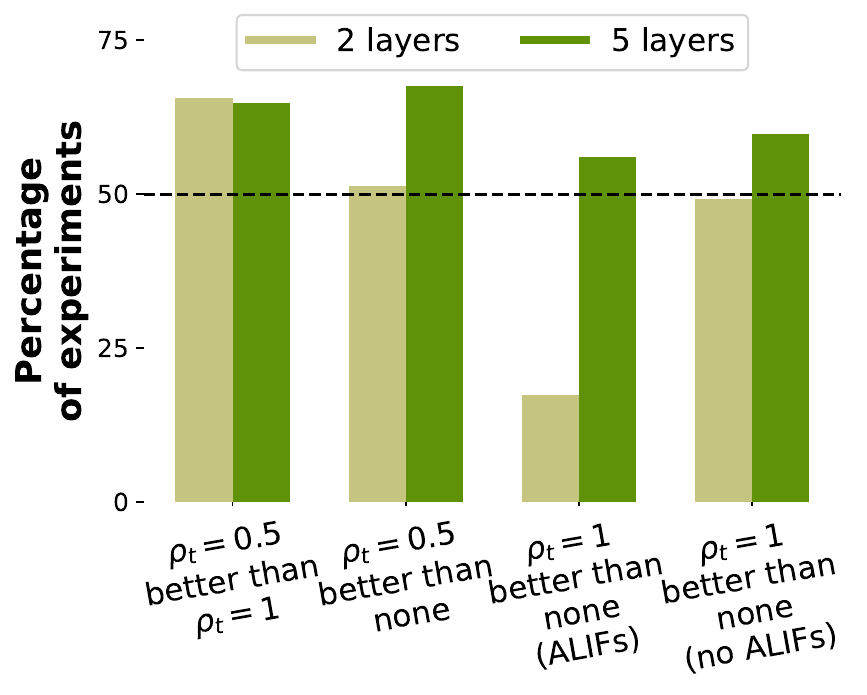}
    \caption{\textbf{Pre-training  to LSC $\rho_t=0.5$ has a stronger impact on deeper $d$-RNNs.} We pre-train and train the $\sigma$-RNN, $ReLU$-RNN, GRU, LSTM networks, on the sl-MNIST, SHD and PTB tasks, for both $\rho_t\in\{0.5, 1\}$, and compare with the non pre-trained case ($none$ in the plot), for 4 random seeds each. We compare $d$-RNN networks with depth $L\in\{2,5\}$, since $\rho_t=0.5$ is expected by our theoretical analysis to have a stronger stabilizing effect for deeper networks. We find indeed that more than 63\% of the times $\rho_t=0.5$ gave better performance than $\rho_t=1$ for both depths. No pre-training matched the rate of $\rho_t=0.5$ for depth $d=2$, but $\rho_t=0.5$ outperformed no pre-training 70\% of the times with deeper networks. With the non differentiable networks ALIF$_+$ and ALIF$_\pm$, it was often impossible to converge to a target  $\rho_t=0.5$, but we see that the deeper the network, the more favorable it was to pre-train to  $\rho_t=1$.}
    \label{fig:rnnlscs}
\end{figure}

\subsection{Pre-training $d$-RNNs to LSC improves training}

Figure \ref{fig:rnnlscs} illustrates the impact on final performance of $d$-RNN network pre-training to satisfy the LSC, and how depth influences the result.  Our results indicate that pre-training for stability improves performance in most scenarios. For differentiable networks, a target radius of $\rho_t=0.5$ consistently produces superior outcomes, over $\rho_t=1$ and over default initialization, and more markedly so for deeper networks. 
On the other side, even if the non-differentiable spiking architectures struggle to converge to $\rho_t=0.5$, they still benefit from a pre-training to $\rho_t=1$.

\subsection{Pre-training State-Space Models to LSC improves training}\label{sec:lruresults}

The Linear Recurrent Unit (LRU) \cite{lru} is one of the most recent models in the family of state-space models, designed to tackle effectively the classification of very long sequences of up to 16K elements. We study the effect of our pre-training to LSC on this model and we call the default initialization the one suggested in the original publication.
We see in Fig.~\ref{fig:lru}, the curves for width 128 networks, after a learning rate grid search with the default initialization, and in Tab.~\ref{tab:lru} after a Gaussian Processes hyper-parameter search over learning rate, width, time steps per batch, dropout, and optimizer with the default initialization, see App.~\ref{app:morenets}. The LSC is applied on the default optimal without any further hyper-parameter search.
A reasonable critique to LSC is, can we just clip the gradients instead, to avoid gradient explosion? However we see in Tab.~\ref{tab:lru} and Fig.~\ref{fig:lru} that gradient clipping does not contribute significantly to performance, if the network is reasonably initialized, while the contribution of the correct initialization is more significant. In fact, clipping typically only rescales the gradient, but  the parameter update still gives an exponential preference to times far in the past, if it did so before clipping. We see that initializing to $\rho_t=1$ gives the worst performance on PTB, for both  3 and 6 depth. We make the distinction between $\rho_t=0.5$ and $\overline{\rho}_t=0.5$, where the $\rho_t=0.5$ means that the radiuses of both time and depth transition derivatives have a mean of 0.5, while $\overline{\rho}_t=0.5$ means that the time component has a mean of $T/(T+L)$ and the depth component a mean of $L/(T+L)$, which still gives an overall mean of 0.5, and still results in Thm.~\ref{thm:ulscboth} bounds. We see that best test perplexity is usually achieved by $\rho_t=0.5$. As expected from our theory, the deeper the network, the stronger the effect of a good initialization. Interestingly, the LRU had already a $\E\rho=0.525$ by default, with higher weight to the time component than to the depth component, which suggests that experimentally the literature is gravitating also towards $\E\rho=0.5$ unwittingly.

\begin{figure}[h]
    \centering
    \includegraphics[width=.5\textwidth]{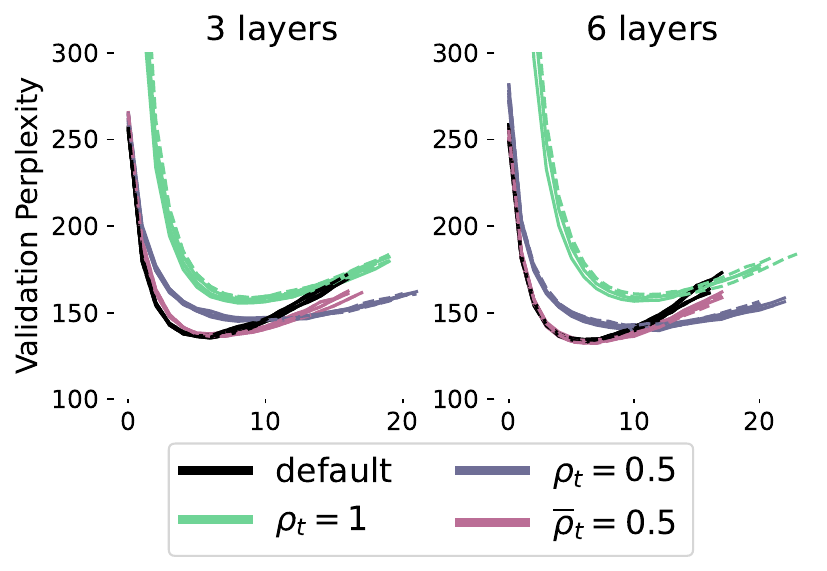}
    \caption{\textbf{Radius of 0.5 outperforms a radius of 1 on state-space models.} We train the LRU state-space model on the PTB task, we see that clipping (dashed) does not have a significant effect, while $\rho_t=1$ is outperformed by $\rho_t=0.5$. Interestingly by default the LRU has an initialization close to $\rho_t=0.5$. We use $\overline{\rho}_t=0.5$ to denote that the time component of the gradient is given a stronger weight than the depth component.}
    \label{fig:lru}
\end{figure}

\begin{table}[h]
\centering
\resizebox{.5\textwidth}{!}{
\begin{tabular}{lcc}
\toprule
 \textbf{Depth 3 LRU} & \textbf{\shortstack{validation \\ perplexity} } & \textbf{\shortstack{test \\ perplexity} } \\
\midrule
default & $139.34~\pm~27.10$ & $110.95~\pm~26.61$ \\
$\rho=1$ & $229.67~\pm~32.02$ & $186.60~\pm~27.56$ \\
$\rho=0.5$ & $131.84~\pm~2.58$ & $103.18~\pm~2.14$ \\
$\overline{\rho}_t=0.5$ & $\textbf{126.57}~\pm~\textbf{1.52}$ & $\textbf{97.94}~\pm~\textbf{1.02}$ \\
default + clip & $139.35~\pm~27.11$ & $110.69~\pm~26.80$ \\
$\rho=1$ + clip & $329.42~\pm~237.82$ & $295.01~\pm~250.27$ \\
$\rho=0.5$ + clip & $\textbf{132.06}~\pm~\textbf{2.30}$ & $\textbf{102.69}~\pm~\textbf{0.60}$ \\
$\overline{\rho}_t=0.5$ + clip & $133.79~\pm~16.02$ & $104.89~\pm~12.46$ \\
\bottomrule
\toprule
 \textbf{Depth 6 LRU} & \textbf{\shortstack{validation \\ perplexity} } & \textbf{\shortstack{test \\ perplexity} } \\
\midrule
default & $\textbf{125.29}~\pm~\textbf{0.82}$ & $\textbf{97.59}~\pm~\textbf{0.78}$ \\
$\rho=1$ & $545.41~\pm~386.85$ & $535.67~\pm~465.11$ \\
$\rho=0.5$ & $132.12~\pm~1.49$ & $103.17~\pm~1.36$ \\
$\overline{\rho}_t=0.5$ & $133.23~\pm~18.00$ & $106.45~\pm~17.02$ \\
default + clip & $147.62~\pm~39.26$ & $119.04~\pm~37.88$ \\
$\rho=1$ + clip & $596.69~\pm~409.26$ & $602.50~\pm~515.43$ \\
$\rho=0.5$ + clip & $\textbf{129.94}~\pm~\textbf{0.46}$ & $\textbf{101.40}~\pm~\textbf{0.48}$ \\
$\overline{\rho}_t=0.5$ + clip & $132.74~\pm~18.52$ & $105.48~\pm~17.96$ \\
\bottomrule
\end{tabular}
}
\caption{\textbf{Pre-training has a stronger effect than gradient clipping.} Validation and test results for the LRU state-space model trained on the PTB language modeling task. A local radius of $1$ worsens performance, while a radius of $0.5$, improves performance, more so also considering that the defaul LRU initialization had a radius of $0.525$, very close to the theoretical $0.5$.}
\label{tab:lru}
\end{table}

%% file: sections/conclusions.tex
\section{Discussion and Conclusions}

Our study has demonstrated the efficacy of pre-training to local stability as a useful strategy for preparing neural networks prior to training. This approach is particularly beneficial when dealing with architectures for which the appropriate initialization hyper-parameters are not known a priori from mathematical analysis. In the traditional approach, the practitioner must determine the optimal initialization analytically \cite{glorot2010understanding,orthogonal_initialization,he2015delving,roberts2022principles}, which can be time-consuming and challenging, especially for complex architectures. Consequently, practitioners often resort to trial and error to identify the best initialization strategy for a given architecture, resulting in repeated training from scratch. In contrast, our pre-training technique is applicable to any architecture, whether old or new, and requires only a single pre-training session to achieve stability, obviating the need for multiple rounds of trial and error.

Secondly, we have extended existing theories \cite{glorot2010understanding,orthogonal_initialization,he2015delving,roberts2022principles, defazio2022scaling, bengio1994learning,hochreiter1997long, hochreiter2001gradient, arjovsky2016unitary, pascanu2013difficulty}, such as Glorot, He, and Orthogonal initializations, to propose the Local Stability Condition (LSC). This new theory allows for the description of the gradient stability of a wide range of multi-layer recurrent networks with minimal assumptions on the distributions of parameters and data. Our approach recovers classical results in fully-connected feed-forward architectures \cite{glorot2010understanding,he2015delving,orthogonal_initialization} and provides a way to work with both upper and lower bounds to the variance of the parameter update. While it is more common to work with upper bounds, we have shown that lower bounds are also possible, although they may not be as easy to stabilize in practice.

Finally, we discovered that $\rho_t=0.5$ is a more desirable target local radius for deep RNNs than $\rho_t=1$. This finding contradicts the conventional line of reasoning \cite{hochreiter1991untersuchungen, bengio1994learning, kerg2019non} that focused on shallow RNNs, and used the time dimension as the effective depth, which leads to only finding multiplicative sources of gradient exponential explosion. Instead, in deep RNNs, depth in combination with time, introduce another source of gradient exponential explosion that is additive and cannot be effectively mitigated by a local radius of $\rho_t=1$, but can be addressed by targeting $\rho_t=0.5$. Given our computational constraints, we were not able to explore deeper networks, which is therefore left for future work, but our experimental results indicate that the deeper the network the more desirable $\rho_t=0.5$ becomes for $d$-RNNs. Interestingly, the $\rho_t=1$ polynomial explosion was observed empirically before \cite{kerg2019non, arjovsky2016unitary}, but to our knowledge we give the first theoretical account for it, by observing that the binomial coefficient bound grows polynomially when depth is fixed.

%% file: sections/acknowledgements.tex
\section{Acknowledgments}

We thank Guillaume Bellec for many careful reads and insights on ALIF networks, and Emmanuel Calvet and Matin Azadmanesh for their feedback on the text. We thank Wolfgang Maass for the opportunity to visit their lab and Y. Bengio for commenting on a first abstract. We thank Terence Tao for suggesting to look at Kostlan Theorem to prove Lemma \ref{thm:kost} and for the key insights to prove Theorem \ref{thm:ssn}. We thank Bilal Piot, Aaron Courville and Razvan Pascanu for their corrections and insights on latest trends.
Luca Herranz-Celotti was supported by the Natural Sciences and Engineering Research Council of Canada through the Discovery Grant from professor Jean Rouat, FRQNT and by CHIST-ERA IGLU. We thank Compute Canada for the clusters used to perform the experiments and NVIDIA for the donation of two GPUs.

%% file: sections/appendix.tex
\onecolumn

\renewcommand{\thesection}{\Alph{section}}
\renewcommand{\theHsection}{A\arabic{section}}

\beginsupplement

\section*{Appendix}

\vspace{5cm}
\section{Proving the upper bounds on the gradient update variance of a $d$-RNN}
\label{app:lsc}

Let's restate here the  general deep RNN system we consider


\begin{equation}
\tag{\ref{eq:general_sys2}}
  \begin{aligned}
    \boldsymbol{h}_{t,l} =& \ g_h( \boldsymbol{h}_{t-1,l}, \boldsymbol{h}_{t-1,l-1},  \boldsymbol{\theta}_l) \\
    \hat{\boldsymbol{o}}_t =& \ g_o(\boldsymbol{h}_{t,L},  \boldsymbol{\theta}_o)
  \end{aligned}
\end{equation}

\noindent and be 

\begin{equation}\begin{aligned}
    a_k \in & \,\bigcup_{t, l}\Big\{\norm{\frac{\partial \boldsymbol{h}_{t, l}}{\partial \boldsymbol{h}_{t-1, l}}}_M, \norm{\frac{\partial \boldsymbol{h}_{t, l}}{\partial \boldsymbol{h}_{t-1, l-1}}}_M\Big\}
\end{aligned}\end{equation}

\noindent where $\norm{\cdot}_M$ is any sub-multiplicative matrix norm \cite{belitskii2013matrix, horn1990}, and therefore a scalar function of a matrix. We introduce the concept of Decaying Covariance, since it will be used at the end of the proof of  Thm.~\ref{thm:lsc}:

\begin{restatable}[Decaying Covariance]{definition}{deccov}
\label{thm:deccov} The set of random variables $\{x_k\}_k$ with $k\in\mathbb{N}$ has Decaying Covariance with $k$ if $\Var x_k<\infty$ and if the covariance of $x_k$ with $x_{k'}$ tends to zero as $|k-k'|\rightarrow\infty$.
\end{restatable}

We show in Fig. \ref{fig:dc1} and \ref{fig:dc2} App. \ref{app:deccov}, that the $\{a_k\}_k$  for the LSTM and the ALIF networks, both with default initialization and after LSC pre-training, on three different tasks, have covariances that are either very close to zero or decay to zero, and thus have Decaying Covariance. This observation supports assuming Decaying Covariance as a reasonable property of most $d$-RNN. Now we proceed to state one of the most important theorems in the article and prove it. This will be the main stepping stone in the proof of Thm.~\ref{thm:ulscboth} in the main text. Even if we only need $q=1$ to prove the remaining Theorems and Corollaries, we prove Thm.~\ref{thm:lsc} for any power of the matrix norm $a_k^q$, for the sake of generality, where $q$ is a real number larger or equal to one.

\begin{restatable}[Local Stability Condition, with matrix norms $\E a_k^q=1$]{theorem}{lsc}
\label{thm:lsc} Be the multi-layer RNN in eq. \ref{eq:general_sys2}. 
The  Local Stability Condition (LSC), $\E a_k^q=1$, for one choice of the real $q\geq1$, is sufficient for an upper bound to the element-wise parameter update variance that in probability increases sub-exponentially, either with time or depth. This result is conditional on $\{a_k^q\}_k$ having Decaying Covariance with increasing distance in time and depth.
\end{restatable}
\begin{proof}
We define a general loss function that compares the output predited by the network $\hat{\boldsymbol{o}}$, to the true output in the dataset $o$, as an average over $T$ time-steps $\overline{\mathcal{L}}(\boldsymbol{o}, \hat{\boldsymbol{o}}) = \frac{1}{T}\sum_t \mathcal{L}(\boldsymbol{o}_t, \hat{\boldsymbol{o}}_t)$ with $t\in [0, \cdots, T]$, where $\mathcal{L}$ can be the mean squared error, cross-entropy or any loss choice.
The average over the samples in a mini-batch is ommitted for cleanliness. We  call
$J^{t,L}_{t',l}=\partial \boldsymbol{h}_{t,L}/\partial \boldsymbol{h}_{t',l}$ the sum of all the chains of derivatives that go from time $t$ and layer $L$, down to time $t'$ and layer $l$. 
Noticing that for a multilayer stack of RNNs, the gradient can only arrive to $t',l$ from one  time step in the future $t+1$, either from one layer above $l+1$ or from the same layer $l$, then $J^{t,L}_{ t',l}$ satisfies what we call, a Grid Derivative Pascal Triangle
\begin{equation}\begin{aligned} \label{eq:gdpt}
    J^{T,L}_{t',l} = J^{T,L}_{t'+1,l}\frac{\partial \boldsymbol{h}_{t'+1,l}}{\partial \boldsymbol{h}_{t',l}} + J^{T,L}_{t'+1,l+1}\frac{\partial \boldsymbol{h}_{t'+1,l+1}}{\partial \boldsymbol{h}_{t',l}}
\end{aligned}\end{equation}
\noindent which comes from $d\boldsymbol{h}_{t,l} = \partial g_h/\partial \boldsymbol{h}_{t-1,l} d \boldsymbol{h}_{t-1,l} +\partial g_h/\partial \boldsymbol{h}_{t-1,l-1}d\boldsymbol{h}_{t-1,l-1}+ \partial g_h/\partial \boldsymbol{\theta}_ld\boldsymbol{\theta}_l$, the layer $l$ differential at time $t$.
As done before for one layer recurrent network \cite{bellec2020solution, metz2021gradients}, we develop the parameter update rule into its sub-components, extending it to the multi-layer setting
\begin{equation}\begin{aligned}
    \Delta\theta_{l} =& -\eta \frac{\partial \overline{\mathcal{L}}}{\partial \boldsymbol{\theta}_{l}}\\
    =& -\frac{\eta}{T}\sum_t\frac{\partial \mathcal{L}}{\partial \boldsymbol{\theta}_{l}}\\
    =& -\frac{\eta}{T}\sum_t\frac{\partial \mathcal{L}}{\partial \hat{\boldsymbol{o}}_t}\frac{\partial\hat{\boldsymbol{o}}_t}{\partial \boldsymbol{\theta}_{l}}\\
    =& -\frac{\eta}{T}\sum_t\frac{\partial \mathcal{L}}{\partial \hat{\boldsymbol{o}}_t}\frac{\partial\hat{\boldsymbol{o}}_t}{\partial\boldsymbol{h}_{t,L}}\frac{d\boldsymbol{h}_{t,L}}{d \boldsymbol{\theta}_{l}}
    \\
    =& \cdots\\
    =& -\frac{\eta}{T}\sum_t\frac{\partial \mathcal{L}}{\partial \hat{\boldsymbol{o}}_t}\frac{\partial\hat{\boldsymbol{o}}_t}{\partial\boldsymbol{h}_{t,L}}
    \Big(
    \sum_{t'\leq t}J^{t,L}_{ t',l}\frac{d\boldsymbol{h}_{t',l}}{d \boldsymbol{\theta}_{l}}
    \Big)\\
    =& -\frac{\eta}{T}\sum_t\sum_{t'\leq t}\frac{\partial \mathcal{L}}{\partial \hat{\boldsymbol{o}}_t}\frac{\partial\hat{\boldsymbol{o}}_t}{\partial\boldsymbol{h}_{t,L}}
    J^{t,L}_{ t',l}\frac{d\boldsymbol{h}_{t',l}}{d \boldsymbol{\theta}_{l}} \label{eq:graddecomp}
\end{aligned}\end{equation}

The following development consists of two major ideas. First, we use matrix norms to turn chains of matrix multiplications into chains of scalar multiplications, as done before to study simple one-layer recurrent architectures in e.g. \cite{bengio1994learning,hochreiter1997long, hochreiter2001gradient, arjovsky2016unitary, pascanu2013difficulty}. Second, we use the arithmetic-geometric mean inequality and a Bernstein Theorem \cite{bernshtein1918loi, kozlov2005weighted, cacoullos2012exercises}, restated below as Theorem \ref{thm:bern}, to bound expectations of correlated functions with expectations of uncorrelated functions. 
It will come at the price of having to analyze an exact upper bound on the variance, an inequality instead of an equality. However, if we find an exact upper bound that does not increase exponentially, nothing that is contained within it can explode exponentially either.

We notice that we can upper bound the element-wise variance of the parameter update with any choice of matrix norm, by (1) considering $\Delta \boldsymbol{\theta}_l$ as a matrix in $\mathbb{R}^{m_l\times 1}$, (2) using the Frobenius norm $\norm{A}_F=(\sum_{ij} a_{ij}^2)^{1/2}$, (3) for small updates $\norm{\Delta \boldsymbol{\theta}_l}_F \leq 1$, and (4) the equivalence of matrix norms, i.e., the fact that any matrix norm can be upperbounded by any other, at the cost of a constant multiplicative factor $c_M$ \cite{golub2013matrix, horn1990}. We therefore have 
\begin{equation}\begin{aligned}
    \Var \Delta \theta_l
    =& \ \E\Delta \theta_l^2-\Big(\E\Delta \theta_l\Big)^2\\
    \leq& \ \E\Delta \theta_l^2\\
    =&\ \frac{1}{m_l}\E[\Delta \boldsymbol{\theta}_l^T\Delta \boldsymbol{\theta}_l] \\
    =& \ \frac{1}{m_l}\E[\norm{\Delta \boldsymbol{\theta}_l}_F^2] && \text{Frobenius definition } \\
    \leq & \ \frac{1}{m_l}\E[\norm{\Delta \boldsymbol{\theta}_l}_F] && \text{small updates } \norm{\Delta \boldsymbol{\theta}_l}_F \leq 1\\
    \leq & \ c_M\frac{1}{m_l}\E[\norm{\Delta \boldsymbol{\theta}_l}_{M}] && \text{norm equivalence}\label{eq:varnorm}
\end{aligned}\end{equation}
We are going to introduce two new indices, $c,k$. As represented in figure \ref{fig:subexp} a), $J$ is the sum of all the derivative chains that go from the last layer and last time step, to the layer and the time step of interest, so the pre-index $c$, will be a reminder of all those chains, and $j$ will represent one of those chains. Then we are going to remind ourselves that each $j$ is composed of the multiplication of the links in the gradient chain. Those links are the transition derivatives of each variable $\boldsymbol{h}_{t,l}$, with respect to every variable it directly depends on, $\boldsymbol{h}_{t-1,l}, \boldsymbol{h}_{t-1,l-1}$. We will call every link in the chain as $M_k$ and the index $k$ will keep track of all the links in the chain $j$. Therefore we define

\begin{equation}\begin{aligned}
    J^{t,L}_{t',l} =& \sum_c \tensor*[_c]{j}{^{t,L}_{t',l}}= \sum_c \prod_k  M_k \label{eq:jasm}\\ 
    M_k\in& \, \bigcup_{t, l}\Big\{\frac{\partial \boldsymbol{h}_{t, l}}{\partial \boldsymbol{h}_{t-1, l}}, \frac{\partial \boldsymbol{h}_{t, l}}{\partial \boldsymbol{h}_{t-1, l-1}}\Big\}
\end{aligned}\end{equation}

We can proceed taking a general matrix norm $\norm{\cdot}_M$, and use the general properties of norms. All norms are  sub-additive, since they have to satisfy the triangle inequality, and we choose among those that are sub-multiplicative, so, applying these properties to Eq.~\ref{eq:graddecomp}, we have

\begin{equation}\begin{aligned}
    \norm{\Delta\theta_{l} }_M
    =& \frac{\eta}{T}\norm{\sum_t\sum_{t'\leq t}\sum_c\frac{\partial \mathcal{L}}{\partial \hat{\boldsymbol{o}}_t}\frac{\partial\hat{\boldsymbol{o}}_t}{\partial \boldsymbol{h}_{t,L}}
    \prod_k M_k\frac{d\boldsymbol{h}_{t',l}}{d \boldsymbol{\theta}_{l}}}_M
    \\
    \leq& \frac{\eta}{T}\sum_t\sum_{t'\leq t}\sum_c \norm{\frac{\partial \mathcal{L}}{\partial \hat{\boldsymbol{o}}_t}\frac{\partial\hat{\boldsymbol{o}}_t}{\partial \boldsymbol{h}_{t,L}}
    \prod_k M_k\frac{d\boldsymbol{h}_{t',l}}{d \boldsymbol{\theta}_{l}}}_M && \text{sub-additivity}
    \\
    \leq& \frac{\eta}{T}\sum_t\sum_{t'\leq t}\sum_c \norm{\frac{\partial \mathcal{L}}{\partial \hat{\boldsymbol{o}}_t}}_M\norm{\frac{\partial\hat{\boldsymbol{o}}_t}{\partial \boldsymbol{h}_{t,L}}}_M
    \prod_k \norm{ M_k}_M\norm{\frac{d\boldsymbol{h}_{t',l}}{d \boldsymbol{\theta}_{l}}}_M && \text{sub-multiplicativity}
\end{aligned}\end{equation}

Let's call $c_{\mathcal{L}oh}$ the following maximal value:

\begin{equation}\begin{aligned}
    c_{\mathcal{L}oh} = \max_{t,t', *}\norm{\frac{\partial \mathcal{L}}{\partial \hat{\boldsymbol{o}}_t}}_M\norm{\frac{\partial\hat{\boldsymbol{o}}_t}{\partial\boldsymbol{h}_{t,L}}}_M
    \norm{\frac{d\boldsymbol{h}_{t',l}}{d \boldsymbol{\theta}_{l}}}_M
\end{aligned}\end{equation}

\noindent where we denote by * a $\max$ across all the data inputs, if the following expectation is over data, across all the parameter initializations if the following expectation is over initializations, or both if the following expectation is over both.

For cleanliness, we rename each term in the product as

\begin{equation}\begin{aligned}
    a_k\in& \,\bigcup_{t, l}\Big\{\norm{\frac{\partial \boldsymbol{h}_{t, l}}{\partial \boldsymbol{h}_{t-1, l}}}_M, \norm{\frac{\partial \boldsymbol{h}_{t, l}}{\partial \boldsymbol{h}_{t-1, l-1}}}_M\Big\}
\end{aligned}\end{equation}

\noindent and applying the Arithmetic-Geometric Mean Inequality (AGMI) \cite{cauchy1821cours, arnold1993four}, followed by Jensen's inequality \cite{jensen1906fonctions}, if $q\geq1$ we have

\begin{equation}\begin{aligned}
    \norm{\Delta\theta_{l}}_M
    \leq& \ c_{\mathcal{L}oh} \frac{\eta}{T}\sum_t\sum_{t'\leq t}\sum_c \prod_k a_k \\
    \leq&  \  c_{\mathcal{L}oh}\frac{\eta}{T}\sum_t\sum_{t'\leq t}\sum_c \Big(\frac{1}{n}\sum_k  a_k\Big)^n && \text{AGMI}\label{eq:AGMI}\\
    =&  \  c_{\mathcal{L}oh}\frac{\eta}{T}\sum_t\sum_{t'\leq t}\sum_c \Big(\frac{1}{n}\sum_k  a_k\Big)^{qn/q} \\
    \leq&  \  c_{\mathcal{L}oh}\frac{\eta}{T}\sum_t\sum_{t'\leq t}\sum_c \Big(\frac{1}{n}\sum_k  a_k^q\Big)^{n/q}  && \text{Jensen}
\end{aligned}\end{equation}

\noindent where $n=t-t'+L-l= \Delta t + \Delta l$. The Jensen's inequality step was not strictly necessary for the purpose of this article, since only the case $q=1$ is required elsewhere, but we make it for the sake of generality.
Assuming  $\{a_k^q\}_k$ having Decaying Covariance, which means that the covariance of $a_k^q$ with $a_{k'}^q$ decreases as the distance in depth and time increases,  $\Cov(a_k^q, a_{k'}^q)\rightarrow0$ as $|k-k'|\rightarrow\infty$, and  provided that $\Var a_k^q<\infty$, we can apply a Weak Law of Large Numbers for dependent random variables, also known as one of Bernstein's Theorems \cite{cacoullos2012exercises}, and restated below. The theorem states that given the mentioned assumptions, the sample average converges in probability to the expected value $\overline{\mu}$, so $\frac{1}{n}\sum_k  a_k^q\xrightarrow{p}\overline{\mu}$ as $n\rightarrow\infty$, where the $p$ over the arrow is a reminder for the convergence in probability. If we do not want it to vanish or explode when exponentiated by $n/q$, we need $\overline{\mu}=1$. A sufficient condition to achieve that is to require $\E a_k^q =1$ since in that case $\E\frac{1}{n}\sum_k  a_k^q=1$. We call $\E a_k^q =1$ the Local Stability Condition (LSC), since it stabilizes the upper bound to the element-wise variance of the parameter update $\Var \Delta\theta$. Therefore, in probability we have
\begin{equation}\begin{aligned}
    \norm{\Delta\theta_{l}}_M
    \overset{p}{\leq}&  \  c_{\mathcal{L}oh} \frac{\eta}{T}\sum_t\sum_{t'\leq t}\sum_c1 + c_p\\
    \E\norm{\Delta\theta_{l}}_M
    \overset{p}{\leq}&  \  c_{\mathcal{L}oh} \frac{\eta}{T}\sum_t\sum_{t'\leq t}\sum_c1  + \overline{c}_p
\end{aligned}\end{equation}
\noindent where we summarize as $c_p$ the fluctuations around the convergence value for small $n$, as $\overline{c}_p=\E c_p$ the average fluctuation for small $n$, and we use the $p$ over the inequality as a reminder that the bound applies in probability. Connecting variance with matrix norm as we showed in Eq.~\ref{eq:varnorm}, we have
\begin{equation}\begin{aligned}
    \Var \Delta \theta_l
    \overset{p}{\leq}  \  \lambda\frac{1}{T}\sum_t\sum_{t'\leq t}\sum_c1  + \phi \\
    \lambda = \ c_{\mathcal{L}oh}\eta\frac{c_M}{m_l}, \quad \phi = c_{\mathcal{L}oh}\frac{c_M\overline{c}_p}{m_l}
\end{aligned}\end{equation}
Substituting the result of Lemma~\ref{thm:binomialgrad} we have that
\begin{equation}\begin{aligned}
    \Var \Delta \theta_l
    \overset{p}{\leq} & \ \lambda\frac{1}{T}\binom{T + \Delta l +2}{T}  + \phi
\end{aligned}\end{equation}
which as we prove in Lemma~\ref{thm:binomialgrad}, is sub-exponential in the two limits (i) $T\rightarrow \infty$, keeping $L$ fixed and (ii) $L\rightarrow \infty$, keeping $T$ fixed, but is exponential in the limit (iii) where both $T, L$ tend to infinity at the same rate. QED
\end{proof}

Instead, the following Theorem studies the effect of setting the expectations of the matrix norms to $0.5$ rather than $1$.

\begin{restatable}[Local Stability Condition, with matrix norms $\E a_k=0.5$]{theorem}{llsc}
\label{thm:llschalfexp} Be the multi-layer RNN in eq. \ref{eq:general_sys2}. A transition derivative with expected matrix norm of $0.5$, results in $\E\norm{\Delta\theta_{l}}_M \leq O(\frac{T+1}{2})$.
\end{restatable}

\begin{proof}
Let's start from the Grid Derivative Pascal Triangle

\begin{equation}\begin{aligned}
    J^{t,l}_{t',l'} = J^{t,l}_{t'+1,l'}\frac{\partial \boldsymbol{h}_{t'+1,l'}}{\partial \boldsymbol{h}_{t',l'}} + J^{t,l}_{t'+1,l'+1}\frac{\partial \boldsymbol{h}_{t'+1,l'+1}}{\partial \boldsymbol{h}_{t',l'}} \label{eq:halfupbound}
\end{aligned}\end{equation}

Take the matrix norm

\begin{equation}
\begin{aligned}
    \norm{J^{t,l}_{t',l'}}_M \leq &\norm{J^{t,l}_{t'+1,l'}}_M\norm{\frac{\partial \boldsymbol{h}_{t'+1,l'}}{\partial \boldsymbol{h}_{t',l'}}}_M + \norm{J^{t,l}_{t'+1,l'+1}}_M\norm{\frac{\partial \boldsymbol{h}_{t'+1,l'+1}}{\partial \boldsymbol{h}_{t',l'}}}_M\\
\end{aligned}
\end{equation}

Let us consider $\hat{J} = \max_{*}\Big\{\norm{J^{t,l}_{t'+1,l' + 1}}_M, \norm{J^{t,l}_{t'+1,l'}}_M\Big\}$ where we denote by * a $\max$ across all the data inputs, if the following expectation is over data, across all the parameter initializations if the following expectation is over initializations, or both if the following expectation is over both. Considering a different weight for the time and depth components, as long as $(\E a_{time} + \E a_{depth})/2 = \E a_k=0.5\implies \E a_{time} + \E a_{depth}=1$, then

\begin{equation}
\begin{aligned}
    \E\norm{J^{t,l}_{t',l'}}_M \leq & \ \E\Big[\norm{J^{t,l}_{t'+1,l'}}_M\norm{\frac{\partial \boldsymbol{h}_{t'+1,l'}}{\partial \boldsymbol{h}_{t',l'}}}_M + \norm{J^{t,l}_{t'+1,l'+1}}_M\norm{\frac{\partial \boldsymbol{h}_{t'+1,l'+1}}{\partial \boldsymbol{h}_{t',l'}}}_M\Big]\\
    \leq & \ \E\Big[\hat{J}\norm{\frac{\partial \boldsymbol{h}_{t'+1,l'}}{\partial \boldsymbol{h}_{t',l'}}}_M + \hat{J}\norm{\frac{\partial \boldsymbol{h}_{t'+1,l'+1}}{\partial \boldsymbol{h}_{t',l'}}}_M\Big]\\
    = & \ \hat{J}\E\Big[\norm{\frac{\partial \boldsymbol{h}_{t'+1,l'}}{\partial \boldsymbol{h}_{t',l'}}}_M\Big] + \hat{J}\E\Big[\norm{\frac{\partial \boldsymbol{h}_{t'+1,l'+1}}{\partial \boldsymbol{h}_{t',l'}}}_M\Big]\\
    = & \ \hat{J}\E a_{time} + \hat{J}\E a_{depth}\\
    = & \ \hat{J}(\E a_{time} + \E a_{depth})\\
    = & \ \hat{J}
\end{aligned}
\end{equation}

\noindent so, the bound on the derivative does not grow in expectation with time and depth. Let's call $c_{\mathcal{L}oh}$ the following maximal value:

\begin{equation}\begin{aligned}
    c_{\mathcal{L}oh} = \max_{t,t', *}\norm{\frac{\partial \mathcal{L}}{\partial \hat{\boldsymbol{o}}_t}}_M\norm{\frac{\partial\hat{\boldsymbol{o}}_t}{\partial\boldsymbol{h}_{t,L}}}_M
    \norm{\frac{d\boldsymbol{h}_{t',l}}{d \boldsymbol{\theta}_{l}}}_M
\end{aligned}\end{equation}

Now we can use it on the bound to the variance of the parameter update with matrix norms as we saw in Eq.~\ref{eq:varnorm}, by using the gradient decomposition from Eq.~\ref{eq:graddecomp}, and with $\lambda = c_{\mathcal{L}oh}\eta\hat{J}$:

\begin{equation}\begin{aligned}
    \E\norm{\Delta\theta_{l}}_M \leq
    & \ \E\norm{\frac{\eta}{T}\sum_t\sum_{t'\leq t}\frac{\partial \mathcal{L}}{\partial \hat{\boldsymbol{o}}_t}\frac{\partial\hat{\boldsymbol{o}}_t}{\partial\boldsymbol{h}_{t,L}}
    J^{t,L}_{ t',l}\frac{d\boldsymbol{h}_{t',l}}{d \boldsymbol{\theta}_{l}}}\\
    \leq
    &  \ \frac{\eta}{T}\sum_t\sum_{t'\leq t}\E\Big[\norm{\frac{\partial \mathcal{L}}{\partial \hat{\boldsymbol{o}}_t}}_M\norm{\frac{\partial\hat{\boldsymbol{o}}_t}{\partial\boldsymbol{h}_{t,L}}}_M
    \norm{J^{t,L}_{ t',l}}_M\norm{\frac{d\boldsymbol{h}_{t',l}}{d \boldsymbol{\theta}_{l}}}_M\Big]\\
    \leq & \ c_{\mathcal{L}oh}\frac{\eta}{T}\sum_t\sum_{t'\leq t}\E
    \norm{J^{t,L}_{ t',l}}_M\\
    \leq & \ c_{\mathcal{L}oh}\frac{\eta}{T}\sum_t\sum_{t'\leq t}\hat{J}\\
    = & \ \hat{J}c_{\mathcal{L}oh}\frac{\eta}{T}\sum_t\sum_{t'\leq t}1\\
    =& \ \lambda\frac{1}{T} T\frac{T+1}{2}\\
    =& \ \lambda\frac{T+1}{2}
\end{aligned}\end{equation}

\noindent which finishes the proof. QED
\end{proof}

We use the previous Theorems to prove the following one, which is the main Theorem stated in the article. We omit some of the assumptions, necessary to prove the previous Theorems, see Decaying Covariance, that should go in the statement of the following, in the interest of readability.

\ulscboth*

\begin{proof}

To prove the result for $\E\rho=1$,   take ($i$) Eq.~\ref{eq:AGMI} Theorem \ref{thm:lsc} with the exponent $q=1$, and use the fact that for any $\epsilon$, there is ($ii$) a sub-multiplicative matrix norm that can be upper-bounded by the matrix radius, such as $\norm{A}_M\leq \rho(A) + \epsilon$ , as proven in \cite{matrixn}. Then ($iii$) for small $\epsilon$, for  $\E\rho_k=1$ then $\frac{1}{n}\sum_k  \E\rho_k\rightarrow 1$ we ($iv$) apply Bernstein Theorem, followed by ($v$) Lemma~\ref{thm:binomialgrad}, and we get

\begin{equation}\begin{aligned}
    \Var \Delta \theta_l
    \leq & \   \frac{\eta c_M}{m_l}\frac{1}{T}\sum_t\sum_{t'\leq t}\sum_c \Big(\frac{1}{n}\sum_k  a_k\Big)^n  && i\\
    \leq & \   \frac{\eta c_M}{m_l}\frac{1}{T}\sum_t\sum_{t'\leq t}\sum_c \Big(\frac{1}{n}\sum_k  \rho_k + \epsilon\Big)^n && ii\\
    = & \  \frac{\eta c_M}{m_l} \frac{1}{T}\sum_t\sum_{t'\leq t}\sum_c \Big(\frac{1}{n}\sum_k  \rho_k\Big)^n + O(\epsilon) && iii\\
    \overset{p}{=} & \  \frac{\eta c_M}{m_l} \frac{1}{T}\sum_t\sum_{t'\leq t}\sum_c 1 + \frac{c_M\overline{c}_p}{m_l} + O(\epsilon) && iv\\
    =& \ \frac{\eta c_M}{m_l}\frac{1}{T}\binom{T + \Delta l +2}{T}  + \frac{c_M\overline{c}_p}{m_l} + O(\epsilon) && v
\end{aligned}\end{equation}

To prove the result for $\E\rho=0.5$ we proceed similarly but taking Thm.~\ref{thm:llschalfexp} as a starting point instead. Take Eq.~\ref{eq:halfupbound}, which describes the bound that the matrix norm imposes on the Grid Derivative Pascal Triangle, and applying the inequality $\norm{A}_M\leq \rho(A) + \epsilon$, for small $\epsilon$ we have:
\begin{equation}\begin{aligned}
    \norm{J^{t,l}_{t',l'}}_M \leq & \ \norm{J^{t,l}_{t'+1,l'}}_M\norm{\frac{\partial \boldsymbol{h}_{t'+1,l'}}{\partial \boldsymbol{h}_{t',l'}}}_M + \norm{J^{t,l}_{t'+1,l'+1}}_M\norm{\frac{\partial \boldsymbol{h}_{t'+1,l'+1}}{\partial \boldsymbol{h}_{t',l'}}}_M\\
    \leq \ &\rho(J^{t,l}_{t'+1,l'})\rho\Big(\frac{\partial \boldsymbol{h}_{t'+1,l'}}{\partial \boldsymbol{h}_{t',l'}}\Big) + \rho(J^{t,l}_{t'+1,l'+1})\rho\Big(\frac{\partial \boldsymbol{h}_{t'+1,l'+1}}{\partial \boldsymbol{h}_{t',l'}}\Big) + O(\epsilon)\\
\end{aligned}\end{equation}
Let us consider $\hat{J} = \max_{*}\Big\{\rho(J^{t,l}_{t'+1,l' + 1}), \rho(J^{t,l}_{t'+1,l'})\Big\}$ where we denote by * a $\max$ across all the data inputs, if the following expectation is over data, across all the parameter initializations if the following expectation is over initializations, or both if the following expectation is over both. If $\E \rho=1/2$ then following a similar reasoning as in Thm.~\ref{thm:llschalfexp} proof, we get
\begin{equation}\begin{aligned}
    \E\norm{J^{t,l}_{t',l}}_M 
    \leq & \ \hat{J} + O(\epsilon)
\end{aligned}\end{equation}
and then using Eq.~\ref{eq:varnorm}, which describes the relation of the variance with matrix norms
\begin{equation}\begin{aligned}
    \Var \Delta \theta_l
    \leq & \ c_M\frac{1}{m_l}\E[\norm{\Delta \boldsymbol{\theta}_l}_{M}]\\
    \leq& \ c_M\frac{1}{m_l}c_{\mathcal{L}oh}\eta\hat{J}\frac{T+1}{2} + O(\epsilon)
\end{aligned}\end{equation}
we obtain the desired result. QED
\end{proof}

\newpage
\clearpage
\section{Equivalence of LSC with classical initializations}
\label{sec:equivalence}
\glohe*

\begin{proof}

Be $M$ a  real $n_l\times n_l$ matrix with all its entries $\E m_{ij}=0$ and $\Var m_{ij}=1$, independent, and $\rho$ its largest eigenvalue modulus.
According to the strong circular law for random matrices \cite{tao2010random}, given a  random matrix with mean zero variance one elements, its largest eigenvalue grows like $\sqrt{n_l}$ as $n_l$ increases. We prove in Lemma \ref{thm:kost} that this is true in expectation also for small matrices, when the elements $ m_{ij}$ are sampled from a radially symmetric distribution, such as the centered gaussian or the centered uniform distributions: the spectral norm satisfies $\E\rho=\sqrt{2}\Gamma((n_l+1)/2)/\Gamma(n_l/2)\approx\sqrt{n_l}$. If the matrix $M$ entries have mean zero and variance different from one, we normalize it such that $M = \sqrt{\Var(m_{ij})}\hat{M}$ where the elements of $\hat{M}$ have variance one. Therefore, since the eigenvalues of a scaled matrix, are scaled by the same quantity (homogeneity), we will have $\E\rho(M) =\sqrt{\Var(m_{ij})}\E\rho(\hat{M})  \approx\sqrt{n_l\Var(m_{ij})}$.

For the linear/ReLU feed-forward network $y_l = W_l\sigma_k(y_{l-1})+b_l$, where in this proof we use $k$ both as an index and as an exponent, with $k=0$ to indicate linear and $k=1$ to indicate ReLU activation, we have $m_{ij}= \partial y_{l, i}/\partial y_{l-1,j}= w_{l,i,j}H(y_{l-1, j})^k$, where $H(y_{l-1, j})$ is the Heaviside function, and we consider $H(y_{l-1, j})^0=1$.
If we use the LSC, $\E\rho(M)=1$, square it, and use the independence of weights and activity, and mean zero weights like in \cite{glorot2010understanding, he2015delving}, we have
$(\E\rho(M))^2=n_l\Var(w_{l,i,j}H(y_{l-1, j})^k)=n_l\Var(w_{l,i,j})\E H(y_{l-1, j})^{2k}=n_l\Var(w_{l,i,j})\E H(y_{l-1, j})^k$, where the last step follows from observing that squaring the Heaviside or one, does not change them. 
Notice that for $k=1$ we are going to use the same extra assumptions as done by \cite{he2015delving}: ``we let $w_{l-1,i,j}$ have a symmetric distribution around zero and $b_{l-1} = 0$, then $y_{l-1}$ has zero mean and has a symmetric distribution around zero".  
Then,  $\E H(y_{l-1, j})^k = 1/2^k$, and we recover both Glorot \cite{glorot2010understanding} and He \cite{he2015delving} scaling for respectively linear and ReLU networks

\begin{equation}\begin{aligned}
    \Var w_l = \frac{2^k}{n_l}
\end{aligned}\end{equation}

\vspace{.3cm}

\noindent since Glorot proposes $ \Var w_l = 1/n_l$ and He $ \Var w_l = 2/n_l$, as we wanted.

Instead, for the Orthogonal initializations in linear networks
\cite{orthogonal_initialization}, we only have to note that an orthogonal $\partial y_{l, i}/\partial y_{l-1,j}=w_{l,i,j}$ satisfies that all its eigenvalues have modulus one, and so will the radius, which is the largest eigenvalue modulus. Therefore the Orthogonal initialization for linear networks satisfies $\E\rho(M)=1$, which is the LSC. QED
\end{proof}

We needed to prove the following Lemma in order to prove the previous Corollary.

\kost*

\begin{proof}
Given the assumptions, we can apply Kostlan theorem \cite{kostlan1992spectra, dubach2018powers}, which states that the modulus of each eigenvalue behaves in distribution as a chi variable, so    $|\lambda_k|\sim \chi_k$. Therefore, in expectation
\begin{equation}\begin{aligned}
    \E|\lambda_k|=\sqrt{2}\frac{\Gamma(\frac{k+1}{2})}{\Gamma(\frac{k}{2})}
\end{aligned}\end{equation}
\noindent the mean of the modulus of the $k$-th eigenvalue follows the mean of the chi distribution, as we wanted to prove, with $\Gamma(\cdot)$ the gamma function. To prove the vanishing variance for the larger eigenvalues, as $k$ increases, notice that given that we are drawing from a $\chi_k$ distribution

\begin{equation}\begin{aligned}
    \Var |\lambda_k| =& \,k-\Big(\E|\lambda_k|\Big)^2\\
     =& \,k-\Big(\sqrt{2}\frac{\Gamma(\frac{k+1}{2})}{\Gamma(\frac{k}{2})}\Big)^2\\
     \xrightarrow{k\rightarrow\infty}&\, k-\Big(\sqrt{k}\Big)^2\\
     =& \,0
\end{aligned}\end{equation}

Therefore the larger the matrix is and the larger the eigenvalue modulus, the less likely it is to deviate from its mean.
QED
\end{proof}

\newpage
\section{Counting the number of gradient paths in a $d$-RNN}
\label{app:counting}

\begin{lemma} \label{thm:binomialgrad}
Be $t, t'\in\{0, 1, \cdots T\}$, $l\in\{0, 1, \cdots L\}$, and $c$ an index for all the shortest paths between opposite ends of a rectangular grid of sides $\Delta l=L-l$ and $\Delta t=t-t'$. The quantity $\frac{1}{T}\sum_t\sum_{t'\leq t}\sum_c1$ is equal to $\frac{1}{T}\binom{T + \Delta l +2}{T}$ and is sub-exponential in the two limits (i) $T\rightarrow \infty$, keeping $L$ fixed and (ii) $L\rightarrow \infty$, keeping $T$ fixed, but is exponential in the limit (iii) where both $T, L$ tend to infinity at the same rate.
\end{lemma}

\begin{proof}

The number of shortest paths from one vertex to the opposite in a grid, is given by the binomial coefficient as follows

\begin{equation}\begin{aligned}
    \frac{1}{T}\sum_t\sum_{t'\leq t}\sum_c1 
    =&  \frac{1}{T}\sum_{t=0}^T\sum_{\Delta t=0}^t \sum_c1\\
    =&  \frac{1}{T}\sum_{t=0}^T\sum_{\Delta t=0}^t \binom{\Delta l + \Delta t}{\Delta t}
\end{aligned}\end{equation}

Using the parallel summation property of binomial coefficients \cite{spivey2019art} we have that

\begin{equation}\begin{aligned}
    \sum_{\Delta t=0}^t \binom{\Delta l + \Delta t}{\Delta t}=&\binom{\Delta l + t+1}{t}\\\
    \sum_{t=0}^T \binom{\Delta l + t+1}{t}=&\binom{\Delta l + T+2}{T}
\end{aligned}\end{equation}

Therefore

\begin{equation}\begin{aligned}
    \frac{1}{T}\sum_t\sum_{t'\leq t}\sum_c1 =&  \ \frac{1}{T}\binom{T + \Delta l +2}{T}
\end{aligned}\end{equation}

Now we can understand how this bound behaves in the limit of depth and time going to infinity. 
We test 3 limits: (i) $T\rightarrow\infty$, with $L$ fixed (ii) $\Delta l\rightarrow\infty$ with $T$ fixed and (iii) $T, \Delta l\rightarrow\infty$ with $T/\Delta l$ fixed.

(i) First $\Delta l\rightarrow\infty$, using the upper bound to the binomial coefficient $\binom{n}{k}\leq(e\cdot n/k)^k$ \cite{binomineq}

\begin{equation}\begin{aligned}
    \frac{1}{T}\binom{T + \Delta l +2}{T}
    \leq \frac{e^T}{T} \Big(\frac{T + \Delta l +2}{T}\Big)^T = O(\Delta l^T)
\end{aligned}\end{equation}

\noindent which is polynomial in $\Delta l$ if $T$ is fixed, therefore sub-exponential, and therefore the bounded formula also has to be sub-exponential.

(ii) Second, we prove the limit $T\rightarrow\infty$ using ($\star$) the binomial coefficient definition, ($\star\star$) using the  $e(n/e)^n\leq n!\leq en(n/e)^n$ upper bound for the numerator, and lower bound for the denominator  \cite{binomineq}, and ($\star\star\star$) considering that $(1 + 1/x)^x<e$ for every $x$ \cite{dorrie2013100, binomineq}, and some arithmetics otherwise

\begin{equation}\begin{aligned}
    \frac{1}{T}\binom{T + \Delta l +2}{T}
    = & \ \frac{1}{T}\frac{(T+\Delta l + 2)!}{(\Delta l + 2)!T!} && \star\\
     \leq& \ \frac{1}{T}\frac{1}{(\Delta l + 2)!}\frac{e(T+\Delta l + 2)\Big(\frac{(T+\Delta l + 2)^{(T+\Delta l + 2)}}{e^{(T+\Delta l + 2)}}\Big)}{eT^Te^{-T}} && \star\star\\
    =& \ \frac{1}{T}\frac{1}{(\Delta l + 2)!}\frac{(T+\Delta l + 2)(T+\Delta l + 2)^{(T+\Delta l + 2)}}{T^Te^{-T}e^{(T+\Delta l + 2)}}\\
    =& \ \frac{1}{T}\frac{1}{(\Delta l + 2)!}\frac{(T+\Delta l + 2)(T+\Delta l + 2)^{(T+\Delta l + 2)}}{T^Te^{(\Delta l + 2)}}\\
    =& \ \frac{1}{T}\frac{1}{(\Delta l + 2)!}\frac{(T+\Delta l + 2)(T+\Delta l + 2)^{T}(T+\Delta l + 2)^{(\Delta l + 2)}}{T^Te^{(\Delta l + 2)}}\\
    =& \ \frac{1}{T}\frac{1}{(\Delta l + 2)!}\frac{(T+\Delta l + 2)T^T(1+\frac{\Delta l + 2}{T})^{T}(T+\Delta l + 2)^{(\Delta l + 2)}}{T^Te^{(\Delta l + 2)}}\\
    =& \ \frac{1}{T}\frac{1}{(\Delta l + 2)!}\frac{(T+\Delta l + 2)(1+\frac{\Delta l + 2}{T})^{\frac{T}{\Delta l + 2}\Delta l + 2}(T+\Delta l + 2)^{(\Delta l + 2)}}{e^{(\Delta l + 2)}}\\
    <& \ \frac{1}{T}\frac{1}{(\Delta l + 2)!}\frac{(T+\Delta l + 2)e^{\Delta l + 2}(T+\Delta l + 2)^{(\Delta l + 2)}}{e^{(\Delta l + 2)}} && \text{$\star\star\star$} \\
    <& \ \frac{1}{T}\frac{1}{(\Delta l + 2)!}(T+\Delta l + 2)^{(\Delta l + 3)}\\
\end{aligned}\end{equation}

\noindent which also is polynomial in $T$ if $\Delta l$ is fixed.
These two limits are sub-exponential. However, we show in the following that if we grow both time and depth together to infinity, the growth is exponential.

(iii) Therefore, we study the third case where both $T, \Delta l\rightarrow\infty$ with $T/\Delta l=k$ fixed. In this case we search for a lower bound since if the lower bound explodes exponentially, nothing above can explode slower than exponentially. If we use ($\star$) the lower bounds $\binom{n}{k}\geq(n/k)^k$ \cite{binomineq} and ($\star\star$) that $(1+ 1/x)^{x+1}>e$ for every $x$ \cite{dorrie2013100, binomineq}, defining the constant $a= 1+ 1/k$  and using some arithmetics

\begin{equation}\begin{aligned}
    \frac{1}{T}\binom{T + \Delta l +2}{T}
    \geq & \ \frac{1}{T} \Big(\frac{T + \Delta l +2}{T}\Big)^T  && \star\\
    = & \ \frac{1}{k\Delta l} \Big(\frac{k\Delta l + \Delta l +2}{k\Delta l}\Big)^{k\Delta l}\\
    = & \ \frac{1}{k\Delta l} \Big(\frac{k +  1}{k}+\frac{2}{k\Delta l}\Big)^{k\Delta l}\\
    = & \ \frac{1}{k\Delta l}a^{k\Delta l}\Big(1+\frac{2}{ak\Delta l}\Big)^{k\Delta l}\\
    = & \ \frac{1}{k\Delta l}a^{k\Delta l}\Big(1+\frac{2}{ak\Delta l}\Big)^{\frac{ak\Delta l}{2}\frac{2}{a}}\\
    = & \ \frac{1}{k\Delta l}a^{k\Delta l}\Big(1+\frac{2}{ak\Delta l}\Big)^{\Big(\frac{ak\Delta l}{2}+1-1\Big)\frac{2}{a}} && \text{preparing for } \star\star\\
    = & \ \frac{1}{k\Delta l}a^{k\Delta l}\Big(1+\frac{2}{ak\Delta l}\Big)^{\Big(\frac{ak\Delta l}{2}+1\Big)\frac{2}{a}}\Big(1+\frac{2}{ak\Delta l}\Big)^{-\frac{2}{a}}\\
    > & \ \frac{e^\frac{2}{a}}{k\Delta l}a^{k\Delta l}\Big(1+\frac{2}{ak\Delta l}\Big)^{-\frac{2}{a}}\label{eq:tlinfstart} && \star\star\\
\end{aligned}\end{equation}

Now, for $x\geq 0$ we have $x/2 > x/e$, and therefore  $\exp(x/2) > \exp(x/e)$, and using Thm.~\ref{thm:exgeqx} proven below, we have that $\exp(x/2) > \exp(x/e) \geq x$. Restating for convenience, this means that $\exp(x/2) >  x$, and if we multiply both sides by $\exp(x/2)/x$ we have that $\exp(x)/x >  \exp(x/2)$. If we substitute in \ref{eq:tlinfstart}, and noticing that the smallest $\Delta l$ is one, we arrive at

\begin{equation}\begin{aligned}
    \frac{1}{T}\binom{T + \Delta l +2}{T}
    > & \ \frac{e^\frac{2}{a}}{k\Delta l}a^{k\Delta l}\Big(1+\frac{2}{ak\Delta l}\Big)^{-\frac{2}{a}}\\
    > & \ e^\frac{2}{a}a^{\frac{k\Delta l}{2}}\Big(1+\frac{2}{ak\Delta l}\Big)^{-\frac{2}{a}}\\
    > & \ e^\frac{2}{a}a^{\frac{k\Delta l}{2}}\Big(1+\frac{2}{ak}\Big)^{-\frac{2}{a}}\\
\end{aligned}\end{equation}

Notice also that $a^k = (1+1/k)^k>1^k=1$ is always greater than one for $k>0$.
Therefore, since $\frac{1}{T}\binom{T + \Delta l +2}{T}$ is lower bounded by a function that grows exponentially with $\Delta l$, it has to grow at least exponentially itself.
We have therefore shown that limits (i) and (ii) are sub-exponential and (iii) is exponential, as we wanted to show. See figure \ref{fig:subexp} for experimental confirmation. QED
\end{proof}

The following small Theorem was necessary to finalize the proof above.

\begin{theorem} \label{thm:exgeqx}
For all real $x$ the following holds: $e^{x/e}\geq x$.
\end{theorem}

\begin{proof}

To prove that $e^{x/e}\geq x$ is true, we will show that $e^{x/e}-x \geq0$ is true, by showing that $f(x) = e^{x/e}-x $ only has one minimum at $f(x)=0$, and the values around are positive. Let's find the extreme values

\begin{align}
    f'(x) =& \frac{1}{e}e^{x/e}-1 = 0\\
    e^{x/e} = e
\end{align}

\noindent which is true only for $x=e$. Since the exponential is monotonic increasing it can only meet the constant $e$ once, so, there is only one solution to $f'(x)=0$. Now, the extreme value is $f(e) = e^{e/e}-e=0$, and going to both infinities $f(x\rightarrow+\infty) \rightarrow+\infty$, $f(x\rightarrow-\infty) \rightarrow+\infty$, both positive, so $x=e$ is a minimum and $f(x)\geq0$. QED
\end{proof}

\newpage
\clearpage
\section{Lower-bounds with anti-norms}
\label{app:lowerbounds}

Our main result is  Thm.~\ref{thm:ulscboth}, which provides an upper-bound to the variance of the parameter update, as it is often done in the $1$-RNN literature. It was only possible because we used two properties of most convenient matrix norms $\norm{\cdot}_M$, sub-additivity and sub-multiplicativity, defined as follows
\begin{align}
    \norm{A+B}_M \leq& \ \norm{A}_M + \norm{B}_M && \text{sub-additivity}\\
    \norm{AB}_M \leq& \ \norm{A}_M\norm{B}_M && \text{sub-multiplicativity}
\end{align}
We noticed that if we reversed these properties we could find a way to a lower bound Theorem. This would mean using an anti-norm \cite{bourin2011norm, ameli2022interpolating, bourin2015anti}, sometimes denoted by $\norm{\cdot}_{!}$, that should satisfy super-additivity and super-multiplicativity, such as 
\begin{align}
    \norm{A+B}_! \geq& \ \norm{A}_! + \norm{B}_! && \text{super-additivity}\\
    \norm{AB}_! \geq& \ \norm{A}_!\norm{B}_! && \text{super-multiplicativity}
\end{align}
Super-additive norms are known as anti-norms \cite{bourin2011norm, ameli2022interpolating, bourin2015anti}, but their study is often confined to Hermitian, positive semi-definite matrices. We prove in Theorem \ref{thm:ssn} below, that anti-norms that are super-multiplicative do not exist in general, but in the restricted case of positive semi-definite matrices, determinants act as super-multiplicative anti-norms. First we prove that with a super-multiplicative anti-norm, a lower bound to the parameter update is possible. Ideally, the anti-norm would capture a notion of dispersion that would correlate with the variance. 
For the sake of generality, we are going to assume an anti-norm has only the properties that we need, super-additivity, super-multiplicativity and homogeneity, meaning $\norm{\alpha A}_! = \alpha \norm{A}_!$ for $\alpha$ a positive scalar. Therefore we do not assume they satisfy anything else, even if for example, they are often defined as positive \cite{guglielmi2020antinorm}.
We name

\begin{equation}\begin{aligned}
    b_k\in& \,\bigcup_{t, l}\Big\{\norm{\frac{\partial \boldsymbol{h}_{t, l}}{\partial \boldsymbol{h}_{t-1, l}}}_!, \norm{\frac{\partial \boldsymbol{h}_{t, l}}{\partial \boldsymbol{h}_{t-1, l-1}}}_!\Big\} \label{eq:bnotation}
\end{aligned}\end{equation}

\noindent which are the anti-norm analogue of the $a_k$ we defined in App.~\ref{app:lsc} for matrix norms.

\begin{restatable}[lower bound Local Stability Condition]{theorem}{llsc}
\label{thm:llsc} Be the multi-layer RNN in eq. \ref{eq:general_sys2}. 
The  lower Local Stability Condition (l-LSC), $\E\log |b_k|=0$, is sufficient for a lower bound to the parameter update anti-norm that in probability composes sub-exponentially, either with time or depth. This result is conditional on $\{\log |b_k|\}_k$ having Decaying Covariance with increasing distance in time and depth.
\end{restatable}

\begin{proof}

We can repeat the steps taken to prove Theorem \ref{thm:lsc}, retrieve the jacobian of the multi-layer recurrent network, and this time apply the super-additivity and super-multiplicativity of our anti-norm. Therefore we start again from the chain rule for the parameter update, Eq.~\ref{eq:graddecomp}, develop it in terms of the transition derivatives $M_k$, Eq.~\ref{eq:jasm}, and apply the two anti-norm properties we are interested in, super-additivity and super-multiplicativity:

\begin{equation}\begin{aligned}
    \norm{\Delta\theta_{l} }_!
    =& \ \frac{\eta}{T}\norm{ \sum_t\sum_{t'\leq t}\sum_c-\frac{\partial \mathcal{L}}{\partial \hat{\boldsymbol{o}}_t}\frac{\partial\hat{\boldsymbol{o}}_t}{\partial \boldsymbol{h}_{t,L}}
    \prod_k M_k\frac{d\boldsymbol{h}_{t',l}}{d \boldsymbol{\theta}_{l}}}_!
    \\
    \geq& \ \frac{\eta}{T}\sum_t\sum_{t'\leq t}\sum_c \norm{- \frac{\partial \mathcal{L}}{\partial \hat{\boldsymbol{o}}_t}\frac{\partial\hat{\boldsymbol{o}}_t}{\partial \boldsymbol{h}_{t,L}}
    \prod_k M_k\frac{d\boldsymbol{h}_{t',l}}{d \boldsymbol{\theta}_{l}}}_! && \text{super-additivity}
    \\
    \geq& \ \frac{\eta}{T}\sum_t\sum_{t'\leq t}\sum_c \norm{-1}_!\norm{\frac{\partial \mathcal{L}}{\partial \hat{\boldsymbol{o}}_t}}_!\norm{\frac{\partial\hat{\boldsymbol{o}}_t}{\partial \boldsymbol{h}_{t,L}}}_!
    \prod_k \norm{M_k}_!\norm{\frac{d\boldsymbol{h}_{t',l}}{d \boldsymbol{\theta}_{l}}}_! && \text{super-multiplicativity}
\end{aligned}\end{equation}

Let's call $c_{\mathcal{L}oh}$ the following minimal value:
\begin{equation}\begin{aligned}
    c_{\mathcal{L}oh} = \min_{t,t', *}\Big|\norm{\frac{\partial \mathcal{L}}{\partial \hat{\boldsymbol{o}}_t}}_M\norm{\frac{\partial\hat{\boldsymbol{o}}_t}{\partial\boldsymbol{h}_{t,L}}}_M
    \norm{\frac{d\boldsymbol{h}_{t',l}}{d \boldsymbol{\theta}_{l}}}_M\Big|
\end{aligned}\end{equation}
\noindent where we denote by * a $\min$ across all the data inputs, if the following expectation is over data, across all the parameter initializations if the following expectation is over initializations, or both if the following expectation is over both.

Using the $b_k$ notation from Eq.~\ref{eq:bnotation}, and tag with $o=1$ the case where the antinorm can take negative values, and $o=0$ the case where the antinorm is always positive, to have

\begin{equation}\begin{aligned}
    \norm{\Delta\theta_{l}}_!
    \geq&  \  c_{\mathcal{L}oh}\frac{\eta}{T}\sum_t\sum_{t'\leq t}\sum_c \prod_k b_k \\
    \geq&  \ (-1)^o c_{\mathcal{L}oh}\frac{\eta}{T}\sum_t\sum_{t'\leq t}\sum_c \prod_k |b_k| \\
    =&  \ (-1)^o c_{\mathcal{L}oh}\frac{\eta}{T}\sum_t\sum_{t'\leq t}\sum_c \exp\Big(\sum_k  \log |b_k|\Big)
\end{aligned}\end{equation}

Assuming Decaying Covariance, which means that the covariance of $ \log |b_k|$ with $ \log |b_{k'}|$ decreases as the distance in depth and time increases,  $\Cov( \log |b_k|,  \log |b_{k'}|)\rightarrow0$ as $|k-k'|\rightarrow\infty$, and  provided that $\Var \log |b_k|<\infty$, we can apply a Weak Law of Large Numbers for dependent random variables, also known as one of Bernstein's Theorems \cite{cacoullos2012exercises}, and restated above. The Theorem states that given the mentioned assumptions, the average converges in probability to the expectation, $\frac{1}{n}\sum_k   \log |b_k|\xrightarrow{p}\overline{\mu}$ as $n\rightarrow\infty$. If we do not want it to vanish or explode when inside the exponential, we need $\overline{\mu}=0$. A sufficient condition to achieve that is to require $\E \log |b_k| =0$ since in that case $\E\frac{1}{n}\sum_k  \log |b_k|=0$. Therefore,   $\E \log |b_k| =0$  stabilizes a lower bound to the anti-norm the parameter update $\Delta\theta$. Therefore, in probability we have

\begin{equation}\begin{aligned}
    \norm{\Delta\theta_{l}}_!
    \overset{p}{\geq}&  (-1)^o\frac{\eta}{T}\sum_t\sum_{t'\leq t}\sum_c1 + c_p\\
    \E\norm{\Delta\theta_{l}}_!
    \overset{p}{\geq}&  (-1)^o\frac{\eta}{T}\sum_t\sum_{t'\leq t}\sum_c1  + \overline{c}_p
\end{aligned}\end{equation}

\noindent where we summarize in $c_p$ the fluctuations around the convergence value for small $n$, and $\overline{c}_p=\E c_p$ the average fluctuation for small $n$. 
As we did for the LSC, Lemma \ref{thm:binomialgrad} shows that this quantity grows exponentially only when both $T,L$ tend to infinity at the same rate, but it is sub-exponential in the cases where only one of them is taken to infinity, leaving the other fixed.

QED
\end{proof}

However, in general an anti-norm that is both super-additive and super-multiplicative, doesn't exist, as we prove in the following Theorem, only using  super-additivity and super-multiplicativity as its properties:

\begin{restatable}[non-existence of super-multiplicative and super-additive scalar matrix functions]{theorem}{ssn}\label{thm:ssn}
No super-additive and super-multiplicative scalar function  exists of complex or real matrices, apart from the zero function.
\end{restatable}


\begin{proof}
Let's define $\norm{\cdot}_!$ a super-additive and super-multiplicative scalar function.
Super-additivity implies $\|0+0\|_! \geq \|0\|_!+\|0\|_!$ so $\|0\|_! \geq 2\|0\|_!$,
which is true only if $\|0\|_! =0$. Super-additivity implies also   $0=\|0\|_! = \|I-I\|_!\geq\|I\| + \|-I\|$ where $I$ is the identity matrix.

The proof for complex matrices is easier, since there is always a matrix such that $-I=JJ$ when all the elements in the diagonal are $i$, the imaginary unit. Therefore by super-multiplicativity $\|-I\|_!=\|JJ\|_!\geq\|J\|_!^2\geq0$ and $\|I\|_!=\|II\|_!\geq\|I\|_!^2\geq0$, which is compatible with the previous statements only if $\|-I\|_!=\|I\|_!=0$. This implies that for every rectangular complex matrix $\|A\|_!=\|AI\|_!\geq\|A\|_!\|I\|_!=0$, but considering super-additivity $0 \geq \|A\|_!+\|-A\|_!$ which is consistent with the previous statement only if $\|A\|_!=0$. This completes the proof for complex matrices.

The proof for real matrices $A\in \R^{m\times n}$ follows the same lines with the complication that there is no such $J$ for $m$ odd. When $m$ is even we can construct a matrix that satisfies $-I=JJ$ by having all the elements in the upper anti-diagonal equal to $1$ and all the elements in the lower anti-diagonal equal to $-1$. Following the same steps outlined for complex matrices we arrive at the conclusion that all matrices $A$ with $m=n$ an even number result in $\|A\|_!=0$. 

The same is true if one of the two sides is even, since we can multiply the even side by the identity matrix to have, if $n$ is even,  $\|A\|_!=\|AI\|_!\geq\|A\|_!\|I\|_!=0$ that combined with super-additivity $0 \geq \|A\|_!+\|-A\|_!$ implies $\|A\|_!=0$. Equally if $m$ is even, by multiplying on the left by the identity. Finally, when both $m,n$ are odd, what still holds is that $\|I\|_!=\|II\|_!\geq\|I\|_!^2\geq0$. For $-I$ we can always construct it as $-I=-A^{-1}A$, with $A\in\R^{m'\times n}$ a rectangular matrix with $m'$ even, that we proved to have $\|A\|_!=0$, so $\|-I\|_!=\|-A^{-1}A\|_!\geq\|-A^{-1}\|_!\|A\|_!=0$. Combined with $0\geq\|I\| + \|-I\|$ it implies $\|I\| = \|-I\|=0$. Now we can repeat the argument used before, $\|A\|_!=\|AI\|_!\geq\|A\|_!\|I\|_!=0$, combined with $0 \geq \|A\|_!+\|-A\|_!$ implies $\|A\|_!=0$ for $m,n$ odd also. QED.
\end{proof}

However, when the matrices considered are positive semi-definite, the determinant is super-additive and (super-)multiplicative. In that case the result of Theorem 
\ref{thm:llsc} can be used to lower bound the non-centered second moment of the parameter update by using e.g. $\E\Delta \theta_{l}^2 =\E\det(\Delta \boldsymbol{\theta}_l^T\boldsymbol{\theta}_l )/m_l$. The determinant of rectangular matrices would be a concern, but solutions could be found using generalized determinants or the Cauchy-Binet formula. Given that the determinant of a matrix is the multiplication of all its eigenvalues, satisfying the lower LSC would mean their multiplication needs to be one, while satisfying both upper and lower LSC means that each eigenvalue has to be equal to one.

\newpage

\newpage
\section{Controlling the local derivatives will control both, backward and forward pass}\label{sec:controlpasses}

As it can be seen below, analogous chains of derivatives describe (1) how much the loss will change $ \overline{\mathcal{L}} + \delta \overline{\mathcal{L}}$ under a small change in the input $x_{t',-1} + \delta x_{t',-1}$, (2) how much the latent representations will change $\boldsymbol{h}_{t,l} + \delta \boldsymbol{h}_{t,l}$ with a small change in the input or a small change in a representation below, and (3) how much the parameters have to change $\theta_{l} + \Delta\theta_{l}$ for the network to get closer to the minimum loss. Therefore, the same chain of multiplied transition derivatives, possibly longer or shorter, describes both, the so called forward and backward pass in gradient back-propagation. The analysis we did to avoid the exponential explosion of the variance of the parameters update through the use of local constraints, will therefore avoid the exponential explosion of  (1) the variance of the change in the loss given a change in the input, (2) the variance of the change in the representations given a change in the input or in another representation and evidently, as we have discussed in most of our Theorems, (3) the variance of the parameter update.

\begin{equation}\begin{aligned}
    \delta \overline{\mathcal{L}}
    =& \frac{1}{T}\sum_t\sum_{t'\leq t}\frac{\partial \mathcal{L}}{\partial \hat{\boldsymbol{o}}_t}\frac{\partial\hat{\boldsymbol{o}}_t}{\partial \boldsymbol{h}_{t,L}}
    J^{t,L}_{t',0}\frac{\partial \boldsymbol{h}_{t',0}}{\partial x_{t', -1}} \delta x_{t', -1}\\
    \delta \boldsymbol{h}_{t, l}
    =& \sum_{t'\leq t}
    J^{t,l}_{t',0}\frac{\partial \boldsymbol{h}_{t',0}}{\partial x_{t', -1}} \delta x_{t', -1}\\
    \delta \boldsymbol{h}_{t, l}
    =& \sum_{t'\leq t}
    J^{t,l}_{t',l'} \delta \boldsymbol{h}_{t', l'}\\
    \Delta\theta_{l} 
    =&-\frac{\eta}{T}\sum_t\sum_{t'\leq t}\frac{\partial \mathcal{L}}{\partial \hat{\boldsymbol{o}}_t}\frac{\partial\hat{\boldsymbol{o}}_t}{\partial \boldsymbol{h}_{t,L}}
    J^{t,L}_{t',l}\frac{d\boldsymbol{h}_{t',l}}{d \boldsymbol{\theta}_{l}}
\end{aligned}\end{equation}

These equations are simple applications of the chain rule for the derivatives to the $d$-RNN we have focused on in this work, that we restate here
\begin{equation}
\tag{\ref{eq:general_sys2}}
  \begin{aligned}
    \boldsymbol{h}_{t,l} =& \ g_h( \boldsymbol{h}_{t-1,l}, \boldsymbol{h}_{t-1,l-1},  \boldsymbol{\theta}_l) \\
    \hat{\boldsymbol{o}}_t =& \ g_o(\boldsymbol{h}_{t,L},  \boldsymbol{\theta}_o)
  \end{aligned}
\end{equation}
\noindent and $\overline{\mathcal{L}}(\boldsymbol{o}, \hat{\boldsymbol{o}}) = \frac{1}{T}\sum_t \mathcal{L}(\boldsymbol{o}_t, \hat{\boldsymbol{o}}_t)$ with $t\in [0, \cdots, T]$, where $\mathcal{L}$ can be the mean squared error, cross-entropy or any loss choice. We define $J^{t,L}_{t',l}=\partial \boldsymbol{h}_{t,L}/\partial \boldsymbol{h}_{t',l}$.

\section{Known results used in this work}

For completeness, we restate the Bernstein Theorem that we use in Theorem \ref{thm:lsc}, as stated in \cite{cacoullos2012exercises}. The original proof is in \cite{bernshtein1918loi, kozlov2005weighted}, and a translation can be found in \footnote{\url{https://math.stackexchange.com/questions/245327/weak-law-of-large-numbers-for-dependent-random-variables-with-bounded-covariance}}. Essentially this Theorem allowed us to make a statement about an expectation of a function $\E f(x)$, using something that was easier to compute, the function of the expectation $f(\E x)$. Notice also that Decaying Covariance was defined at the beginning of App.~\ref{app:lsc} to be able to use this Bernstein Theorem.

\begin{restatable}[Bernstein Theorem]{theorem}{bern}
\label{thm:bern}
Let $\{X_n\}$ be a sequence of random variables so that $\Var(X_i)<\infty$, $i=1, 2, \cdots$ and $Cov(X_i, X_j)\rightarrow0$ when $|i-j|\rightarrow\infty$, then the Weak Law of Large Numbers (WLLN) holds.
\end{restatable}

\noindent where the WLLN simply states that the sample average converges in probability towards the expected value \cite{loeve1977elementary}: $\overline {X}_{n}\ {\xrightarrow {p}}\ \mu$ when $n\to \infty$. Equivalently, for any positive number $\varepsilon$,

\begin{equation}\begin{aligned}
    \lim _{n\to \infty }\Pr \Big(\,|{\overline {X}}_{n}-\mu |<\varepsilon\,\Big)=1
\end{aligned}\end{equation}

\newpage
\section{More Training Details}
\label{app:trainingdetails}

The FFNs we use have a width of $128$, and we train until convergence on the validation set with a batch size of $32$.
We collect the training hyper-parameters for the RNNs in Tab.~\ref{tab:traindetails}.
The learning rates were chosen after a grid search over the set $\{10^{-2},3.16 \cdot10^{-3}, 10^{-3}, 3.16 \cdot10^{-4}, 10^{-4}, 3.16 \cdot10^{-5},10^{-5 }\}$ with the ALIF and LSTM networks on each task, shown in Fig.~\ref{fig:lrs}. We used the LSTM optimal learning rate on all the differentiable networks, while we used the ALIF optimal on the non-differentiable networks.
On the PTB task, the input passes through an embedding layer before passing to the first layer, and the output of the last layer is multiplied by the embedding to produce the output, removing the need for the readout
\cite{wozniak2020deep, radford2018improving}. We show in Tab.~\ref{tab:traindetails} the hyper-parameters used for a depth of $L=2$, where the widths are chosen to keep the number of learnable parameters comparable across neuron models. Whenever we report different depths, the width is kept as in the $L=2$ case.
\begin{table}[h]
\centering
\begin{tabular}{rccc}
& sl-MNIST & SHD &  PTB \\ \hline
\\
batch size & 256 & 100 & 32 \\
weight decay & None & None &  None \\
gradient clipping & None & None & None \\
\shortstack{train/val/test \\ \ } & \shortstack{45k/5k/10k\\ samples } & \shortstack{8k/1k/2k \\ samples }  &   \shortstack{930k/74k/82k \\ words} \\
time step repeat & 2 & 2 & 2 \\
\\
LSTM\\
layers width $n_l$ & 42/42 & 94/94 & 500/300 \\
parameters & 153,646 & 371,884 &  5,563,200 \\
learning rate $\eta$ & $10^{-2}$ &  $10^{-3}$ &  $3.16 \cdot 10^{-4}$ \\
\\
GRU\\
layers width $n_l$ & 53/53 & 117/117 & 625/300 \\
parameters & 151,113 & 372,665 &  5,572,425 \\
\\
RNN $\sigma/ReLU$\\
layers width $n_l$ & 128/128 & 256/256 &  1300/300 \\
parameters & 151,050 & 381,460 &  5,561,600 \\
\\
ALIF$_{+/\pm}$\\
layers width $n_l$ & 128/128 & 256/256 &  1300/300 \\
parameters & 151,820 & 382,998 &  5,566,400 \\
$\hat{\boldsymbol{\tau}}_y$ & 0.1 & 0.1 & 0.1\\  $\hat{\boldsymbol{\tau}}_\vartheta$  & 100 & 250 & 100\\
$\hat{\boldsymbol{b}}^{\vartheta}_l$  & 0.01 & 0.01 & 0.01\\
$\hat{\boldsymbol{\beta}}_l$  & 1.8 & 1.8 & 1.8\\
learning rate $\eta$ & $10^{-3}$ &  $10^{-3}$ &  $10^{-3}$ \\
\end{tabular}
\caption{\textbf{Hyper-parameters used for the experiments.} Different hyper-parameters are used for different tasks. Importantly, widths are chosen to have a comparable number of parameters across architectures. Learning rates for optimal performance on LSTM and ALIF, after the grid search shown in Fig.~\ref{fig:lrs}. No gradient clipping nor weight decay are applied, to better show the effect of gradient explosion. }\label{tab:traindetails}
\end{table}
\begin{figure}[h]
    \centering
    \includegraphics[width=\textwidth]{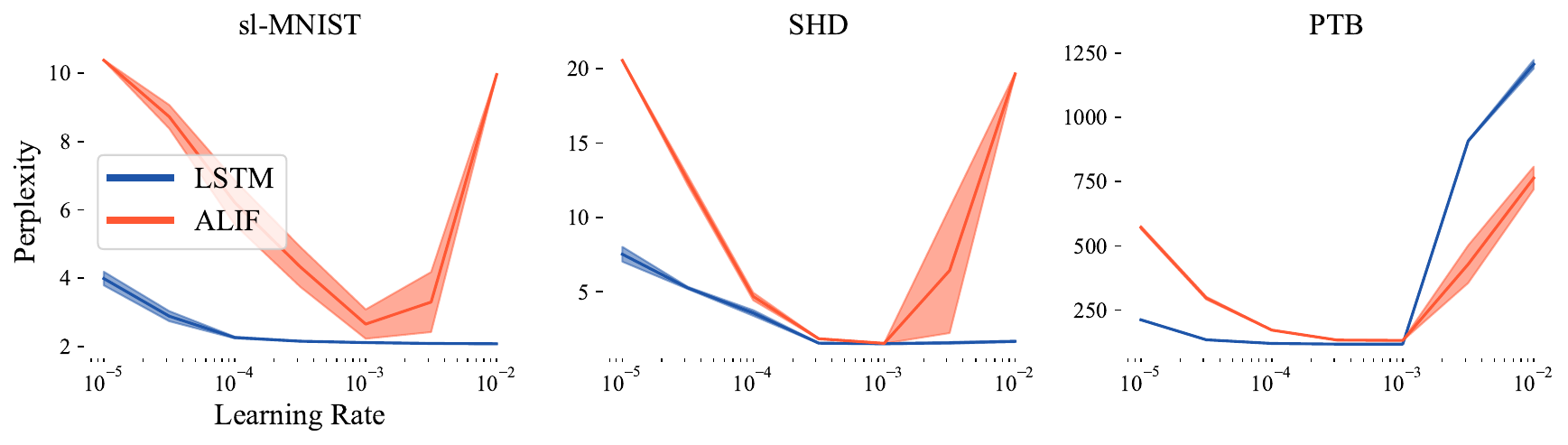}
    \caption{\textbf{Grid Search for best learning rate.} We show in the plots perplexity against learning rate, where perplexity is the exponentiation of the cross-entropy loss, to have more homogeneous plots. Grid search for optimal learning rate was performed on LSTM and ALIF. Optimal LSTM learning rate was used on the differentiable architectures, and optimal ALIF learning rate was used on the non-differentiable architectures.}
    \label{fig:lrs}
\end{figure}

\newpage
\section{Neural Networks Details}
\label{app:morenets}

We pre-train six different recurrent networks to satisfy the LSC that we describe in more detail in the following.

\textbf{LSTM.} The LSTM implementation we used is the following
\begin{equation}\begin{aligned}
    \boldsymbol{i}_t =& \, \sigma_g(W_i\boldsymbol{x}_t + U_i\boldsymbol{h}_{t-1}+\boldsymbol{b}_i)\\
    \boldsymbol{f}_t =& \,  \sigma_g(W_f\boldsymbol{x}_t + U_f\boldsymbol{h}_{t-1}+\boldsymbol{b}_f)\\
    \boldsymbol{o}_t =& \,  \sigma_g(W_o\boldsymbol{x}_t + U_o\boldsymbol{h}_{t-1}+\boldsymbol{b}_o)\\
    \tilde{\boldsymbol{c}}_t =& \,  \sigma_c(W_c\boldsymbol{x}_t + U_c\boldsymbol{h}_{t-1}+\boldsymbol{b}_c)\\
    \boldsymbol{c_t} =& \,  \boldsymbol{f}_t\circ \boldsymbol{c}_{t-1}+ \boldsymbol{i}_t\circ \tilde{\boldsymbol{c}}_t\\
    \boldsymbol{h}_t =& \,  \boldsymbol{o}_t\circ\sigma_h(\boldsymbol{c}_t)
\end{aligned}\end{equation}
The dynamical variables $\boldsymbol{i}_t, \boldsymbol{f}_t, \boldsymbol{o}_t$ represent the input, forget and output gates, while $\boldsymbol{c}_t, \boldsymbol{h}_t$ represent the two hidden layers of the LSTM, that act as a working memory. The matrices $W_j$ are initialized as Glorot Uniform, $U_j$ as Orthogonal, and the biases $\boldsymbol{b}_j$ as zeros, with $j\in\{i, f, o, c\}$. The activations are the sigmoid for $\sigma_g$ and hyperbolic tangent for $\sigma_c, \sigma_h$.

\textbf{GRU.} The GRU implementation we used is the following
\begin{equation}\begin{aligned}
    \boldsymbol{z}_t =& \, \sigma_g(W_z\boldsymbol{x}_t + U_z\boldsymbol{h}_{t-1}+\boldsymbol{b}_z)\\
    \boldsymbol{r}_t =& \,  \sigma_g(W_r\boldsymbol{x}_t + U_r\boldsymbol{h}_{t-1}+\boldsymbol{b}_r)\\
    \hat{\boldsymbol{h}}_t =& \,  \phi_h(W_h\boldsymbol{x}_t + U_h(\boldsymbol{r}_t\circ\boldsymbol{h}_{t-1})+\boldsymbol{b}_h)\\
    \boldsymbol{h_t} =& \, (1-\boldsymbol{z}_t)\circ \boldsymbol{h}_{t-1}+ \boldsymbol{z}_t\circ \hat{\boldsymbol{h}}_t
\end{aligned}\end{equation}
The dynamical variables $\boldsymbol{r}_t, \boldsymbol{z}_t$ represent the  gates, while $\boldsymbol{h}_t$ represents the hidden layer, that acts as a working memory. The matrices $W_j$ are initialized as Glorot Uniform, $U_j$ as Orthogonal, and the biases $\boldsymbol{b}_j$ as zeros, with $j\in\{z,r,h\}$. The activations are $\phi_h$ the hyperbolic tangent and $\sigma_g$ the sigmoid.

\textbf{RNN $\sigma/ReLU$.} The RNN implementation we used is $\boldsymbol{h}_{t,l} = a(W_{rec,l}\boldsymbol{h}_{t-1,l} + W_{in,l}\boldsymbol{h}_{t-1,l-1}+ \boldsymbol{b}_l)$ with either $a$ being the sigmoid activation or ReLU. On the PTB task, we use a smooth ReLU instead of ReLU itself, Swish \cite{hendrycks2016gaussian, ramachandran2017searching}, since ReLU gave errors pre-training. The matrices $W_{in,l}$ are initialized as Glorot Uniform, $W_{rec,l}$ as Orthogonal, and the biases $\boldsymbol{b}_j$ as  Glorot Uniform, to make sure the neurons are not off when the input is zero.

\textbf{Biologically Plausible Adaptive Spikes (ALIF).} Computationally, the simplest model of a spiking biological neuron is the  Leaky-Integrate-and-Fire (LIF) neuron \cite{lapique1907recherches}. When extended with a dynamic threshold \cite{brette2005adaptive}, it is known to be able to reproduce qualitatively all major classes of neurons, as defined electrophysiologically in vitro  \cite{izhikevich2003simple}, and to reproduce up to 96\% precision the spiking time of more detailed models of neurons \cite{brette2005adaptive}. A common definition of an Adaptive LIF (ALIF) \cite{lsnn, yin2021accurate, gerstner2014neuronal} is:
\begin{equation}\label{eq:alif}
\begin{aligned}
    \boldsymbol{y}_{t,l} =& \, \boldsymbol{\alpha}^y_{l} \boldsymbol{y}_{t-1,l} \\&+ W_{rec,l}\boldsymbol{x}_{t-1,l}  + W_{in,l}\boldsymbol{x}_{t-1,l-1}\\
    &+\boldsymbol{b}^y_l-\boldsymbol{\vartheta}_{t-1,l}\boldsymbol{x}_{t-1,l}  \\
    \boldsymbol{\vartheta}_{t,l} =& \, \boldsymbol{\alpha}^{\vartheta}_{l} \boldsymbol{\vartheta}_{t-1,l} +\boldsymbol{b}^{\vartheta}_l + \boldsymbol{\beta}_l \boldsymbol{x}_{t-1,l} \\
    \boldsymbol{x}_{t,l} =& \, H(\boldsymbol{y}_{t,l} - \boldsymbol{\vartheta}_{t,l}) \\
    \hat{\boldsymbol{o}}_t =& \, W_R \boldsymbol{x}_{t,L}
\end{aligned}
\end{equation}
\noindent where $\boldsymbol{y}_{t,l}, \boldsymbol{\vartheta}_{t,l}$ are the dynamical variables that describe the voltage and the threshold of the neuron. A spike $\boldsymbol{x}_{t,l}$ is generated when the voltage surpasses the threshold, and the term $-\boldsymbol{\vartheta}_{t-1,l}\boldsymbol{x}_{t-1,l}$ represents a soft reset mechanism. Neurons across layers are connected by $W_{in,l}\in\mathbb{R}^{n_l\times n_{l-1}}$ while within by $W_{rec,l}\in\mathbb{R}^{n_l\times n_l}$. The parameters, $\boldsymbol{b}^y, \boldsymbol{b}^\vartheta$ set the equilibrium voltage and threshold,  $\boldsymbol{\alpha}^y, \boldsymbol{\alpha}^\vartheta$ set the speed of change of both dynamical variables, and the spike frequency adaptation $\boldsymbol{\beta}$ determines how fast the threshold rises after the emission of a spike \cite{sfa_darjan, language_spike_adaptation}. For both $j\in\{y, \vartheta\}$ we computed $\boldsymbol{\alpha}^j=\exp(-1/\boldsymbol{\tau}_j)$. At initialization $W_{rec,l}, W_{in,l}$, are sampled as Glorot Uniform, $\boldsymbol{b}^y_l$ is set equal to zero and non trainable, while $\mu\in\{\boldsymbol{\tau}_y, \boldsymbol{\tau}_\vartheta, \boldsymbol{b}^{\vartheta}_l,\boldsymbol{\beta}_l \}$ take values in line with \cite{lsnn}, are trainable, and each sampled at initialization from a truncated Gaussian of mean $\hat{\mu}$ and standard deviation $3\hat{\mu}/7$, that guarantees positive values, see App.~\ref{app:trainingdetails}. We call ALIF$_{\pm}$ an alternative initialization with the only difference that $\beta$ is initialized as a gaussian centered at zero, with standard deviation $\hat{\mu}/n_{in}$. During pre-training $\{\boldsymbol{b}^{\vartheta}_l,\boldsymbol{\beta}_l \}$ are adjusted with the input $\kappa$ multiplier, while the $\{\boldsymbol{\tau}_y, \boldsymbol{\tau}_\vartheta\}$ are pre-trained with the recurrent $\kappa$ multiplier.

The $H$ in the ALIF equations is a Heaviside function, and therefore non differentiable, unless distributional derivatives are considered, in which case the delta function represents its derivative. However the delta function is zero almost everywhere, so the derivative of the fast-sigmoid is used in the backward pass as a surrogate gradient \cite{zenke2018superspike}. The surrogate gradient  is defined as $d\tilde{H}(v)/dv = \gamma f(\omega\cdot v)$, where $\omega$ is the sharpness, $\gamma$ the dampening, $f$ is the shape of choice and $\cdot$ the scalar product. We use a $\omega=1$ and  $\gamma=0.5$, in line with \cite{lsnn} and because we saw better performances. Our theory still applies, since, even if we approximate the gradient with the surrogate gradient, the surrogate gradient is the exact formula that we use to propagate information in the backward pass and update the weights. In terms of biological plausibility, the surrogate gradient gives us a higher quality estimate of a good direction to update the parameters of the network than the true gradient, that is zero almost always. Therefore, a desirable biologically plausible update rule, should make similar choices about the update direction as the surrogate gradient update.

\textbf{LRU.} The stem of the LRU is defined as

\begin{align}
    x_k =& \ \Lambda x_{k-1} + \exp(\gamma^{\log}) \odot (Bu_k)\\
    y_k =& \ \Re(Cx_k) + Du_k
\end{align}

\noindent where we call $n_h$ the width of the hidden state, $u_k\in\mathbb{R}^{n_l}$ is the input, $B\in\mathbb{C}^{n_l\times n_h}$ where the real and imaginary components are initialized as gaussians with $1/\sqrt{n_l}$ standard deviation, $C\in\mathbb{C}^{n_l\times n_h}$ where the real and imaginary components are initialized as gaussians with $1/\sqrt{2n_h}$ standard deviation, $D\in\mathbb{R}^{n_l}$ initialized as gaussians with $1$ standard deviation. The matrix $\Lambda\in\mathbb{C}^{n_h\times n_h}$ only has diagonal learnable parameters, of the form $\lambda_j=\exp(-\exp(\nu^{\log}) + i\theta_j)$ with the radius $\exp(-\exp(\nu^{\log}))$ uniformly distributed between $[0.5, 1]$ and $\theta_j$ uniformly distributed between $[0, \pi/10]$. To improve performance, the input is passed through a layer normalization \cite{ba2016layer}, then fed to the LRU stem, that goes through a GeLU activation \cite{hendrycks2016gaussian}, followed by a dropout layer \cite{srivastava2014dropout} of $0.1$ dropout probability, a GLU layer \cite{dauphin2017language}, another dropout, and added back to the input. We add an additional dropout layer before the layer norm and the skip connection, that we call the exterior dropout in the next paragraph.

When we mention the grid search on the learning rate in Sec.~\ref{sec:lruresults},
for a width of 128, 100 time steps per batch, and an exterior dropout of $0$, the best learning rate found was $1e^{-2}$. When we run the Gaussian Process hyper-parameter optimization, the best combination found was an exterior dropout of $0.1$ for a depth of 6, and of $0.2$ for a depth of 3, a width of $256$, a learning rate of $3e^{-2}$ with cosine decay to a tenth of that and no restarts, batch size of $64$, $300$ time steps per batch, AdaBelief optimizer \cite{zhuang2020adabelief}  as a Lookahead Optimizer \cite{zhang2019lookahead}
and with Stochastic Weight Averaging \cite{Izmailov2018AveragingWL}. We also found that a modification of the Masked Language Model technique \cite{bertmodel} improved learning. For the initial 3 epochs of learning, we randomly switch 5\% of words at the input to another word, but the target words remain the correct ones in the dataset.

\newpage
\section{Decaying covariance of ALIF and LSTM in more tasks}
\label{app:deccov}

To prove Theorem \ref{thm:lsc}, we need to assume that $\{a_k^q\}_k$ has decaying covariance with increasing distance in time and depth, where $q$ is a real number such that $q\geq 1$. Since we only use a depth of 2, the more significant multiplicative exponential composition will come from the increment in time steps. Here we show that experimentally $\{a_k^q\}_k$ has decaying covariance, for $q=1$ and with the matrix norm induced by the vector $p$-norm with $p=2$. We show how the decaying covariance condition is satisfied by the ALIF and LSTM  networks, with their default initializations in Fig.~\ref{fig:dc1} and after pre-training to satisfy the LSC in Fig.~\ref{fig:dc2}, on all the tasks considered in this study. 

\begin{figure}[h]
    \centering
    \includegraphics[width=.45\textwidth]{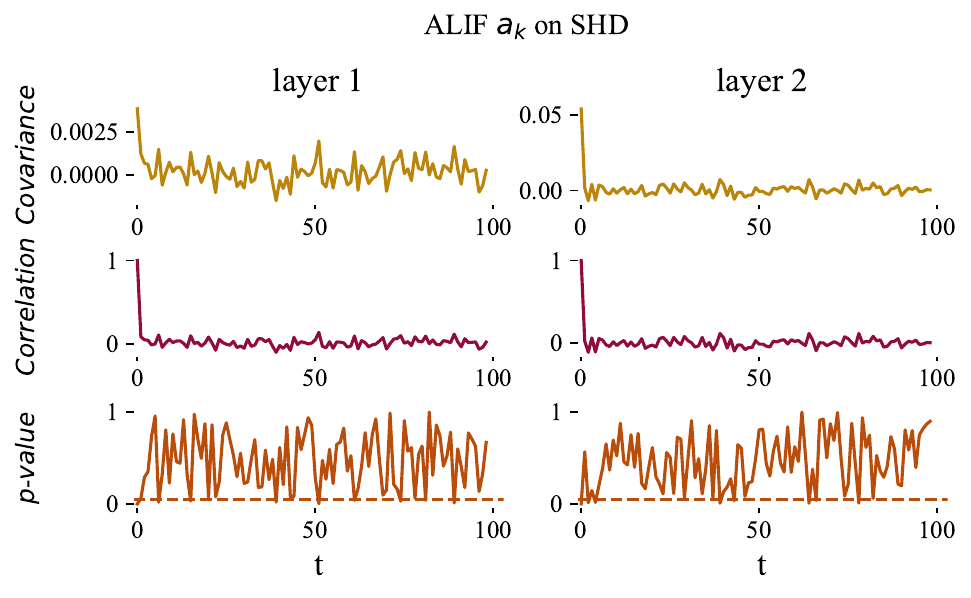}
    \includegraphics[width=.43\textwidth]{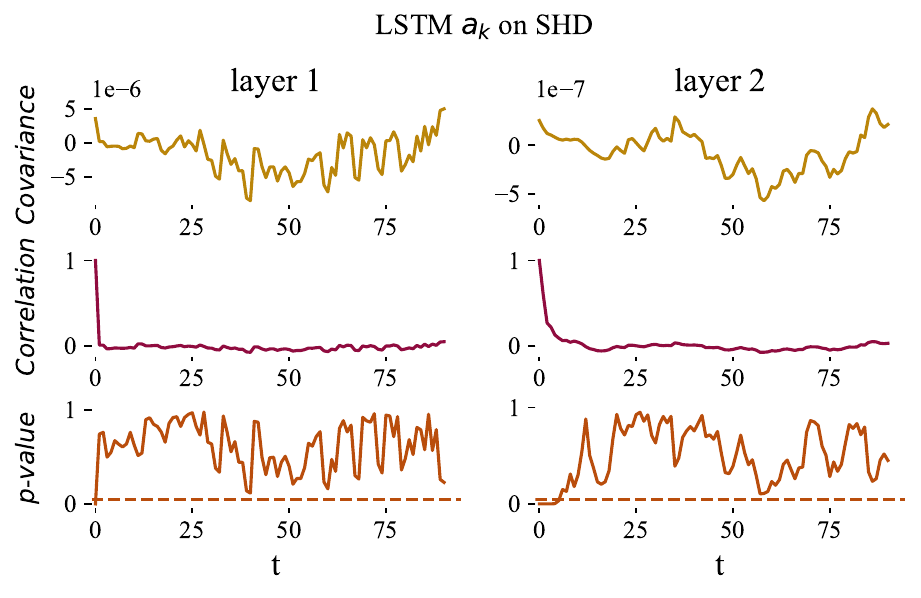}
    \includegraphics[width=.43\textwidth]{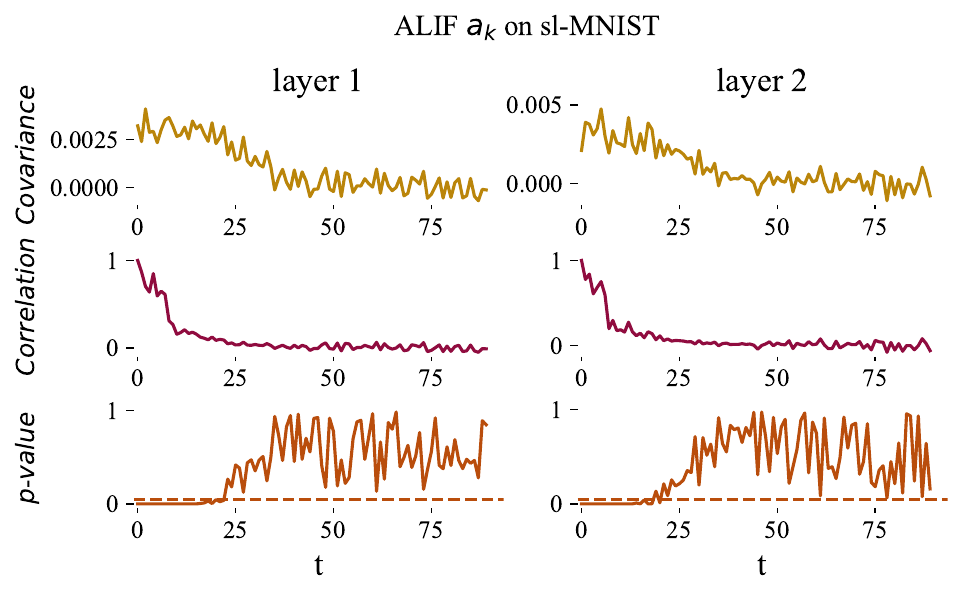}
    \includegraphics[width=.43\textwidth]{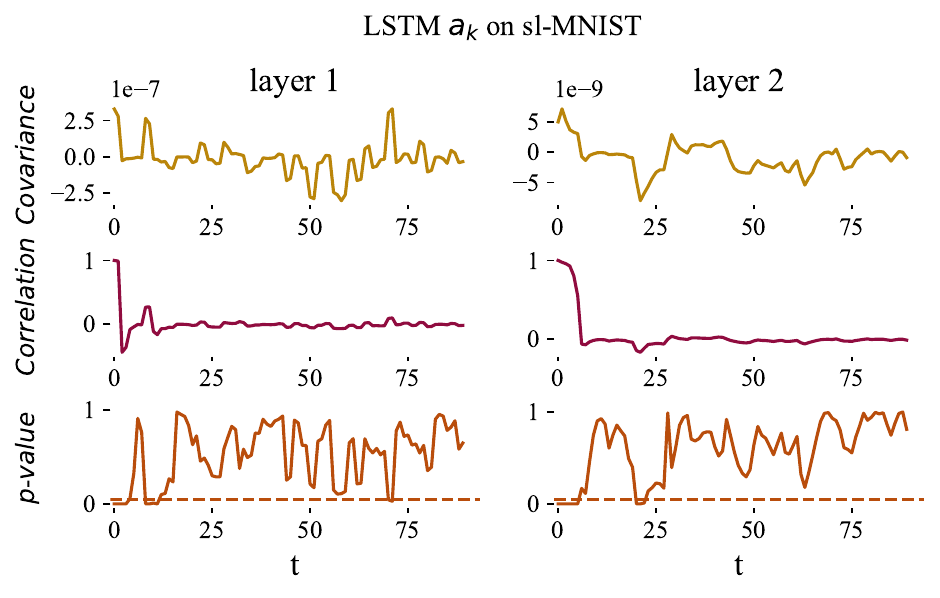}
    \includegraphics[width=.45\textwidth]{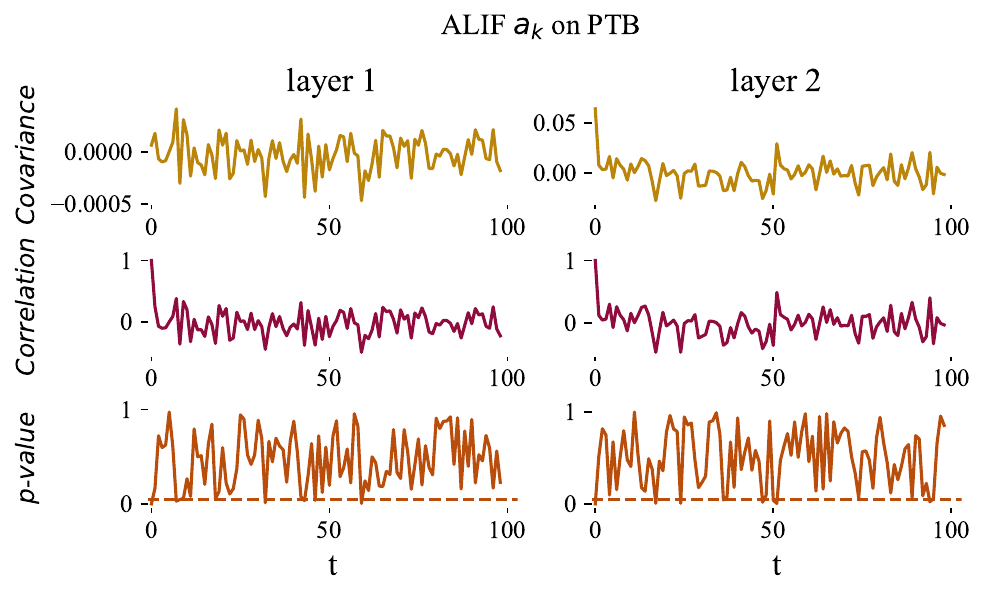}
    \includegraphics[width=.43\textwidth]{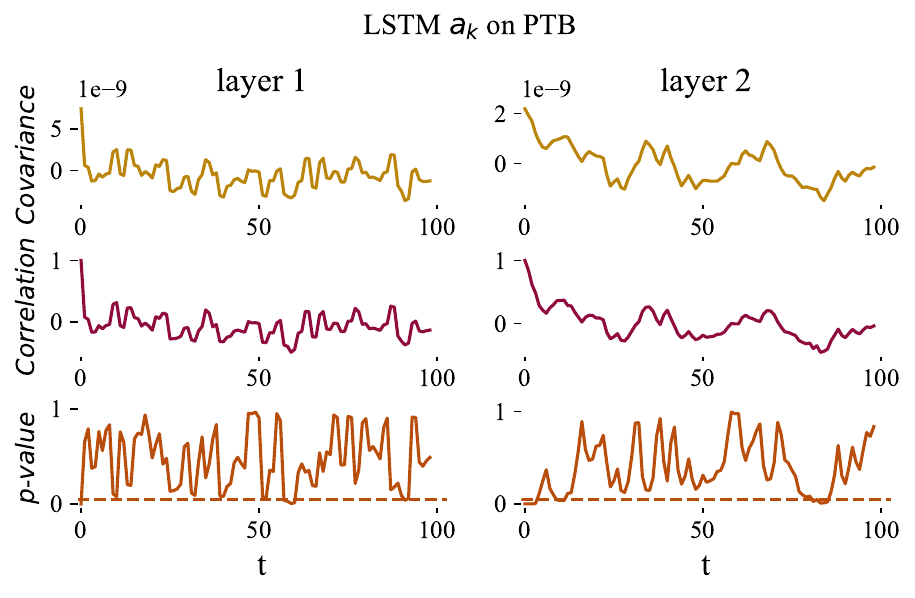}
    \caption{\textbf{Covariance of  $a_k$ decays with time or remains very close to zero at initialization.} The three statistics, covariance, correlation and $p$-value of $a_k$, are computed at each data time step $t$ with respect to $t=0$, with \textit{numpy} \cite{harris2020array}. In this case, the $a_k$ chosen is the matrix norm induced by the vector $p$-norm with $p=2$. We can see that the covariances are very small. We can also see that the correlation quickly decays to zero, and the $p$-values above $0.05$ indicate that the small correlations are not significant. In this plot LSTM and ALIF were initialized with their default initialization. This shows that both ALIF and LSTM have Decaying Covariance at initialization, as defined in App.~\ref{app:lsc} and as necessary to prove Thm.~\ref{thm:lsc}.}
    \label{fig:dc1}
\end{figure}

\begin{figure}
    \centering
    \includegraphics[width=.45\textwidth]{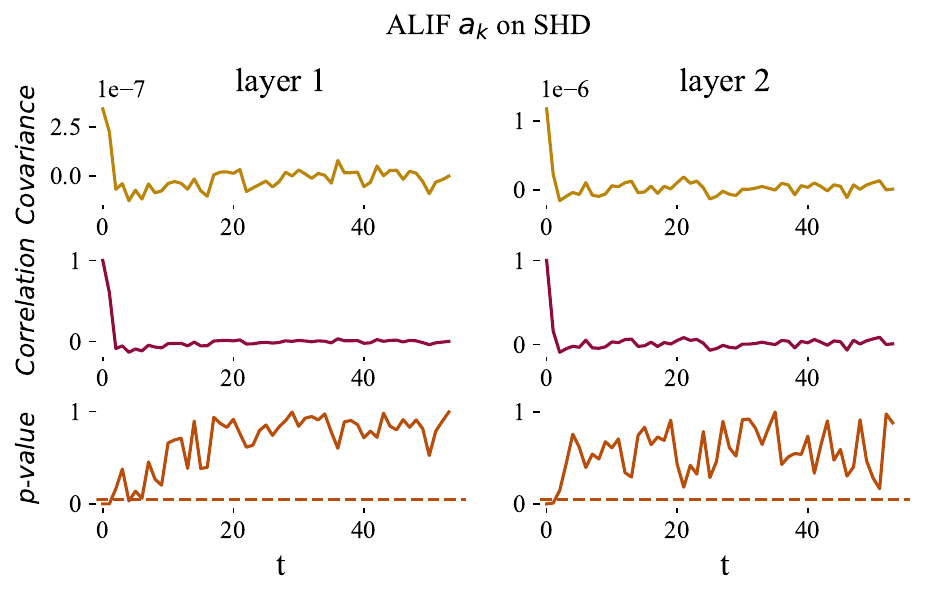}
    \includegraphics[width=.43\textwidth]{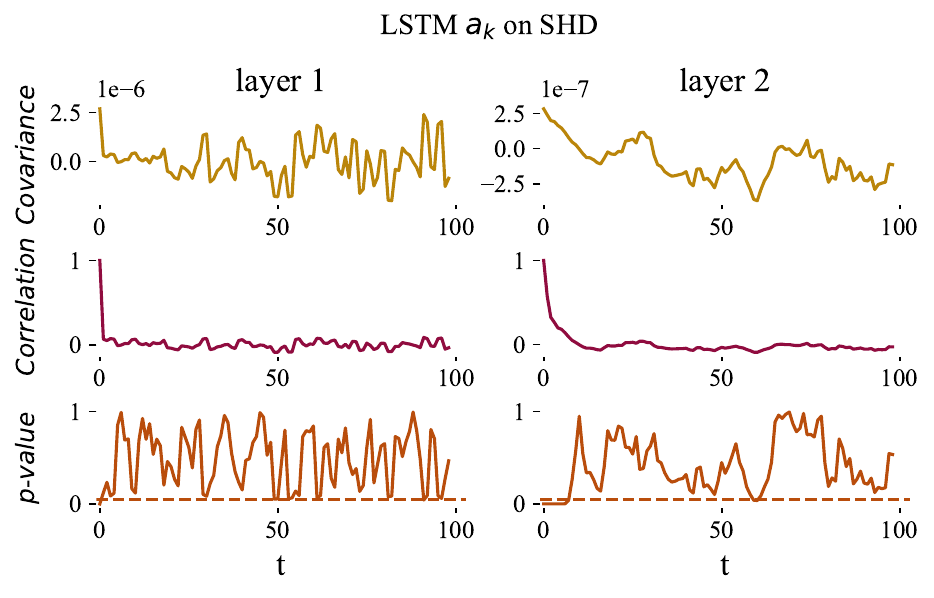}
    \includegraphics[width=.43\textwidth]{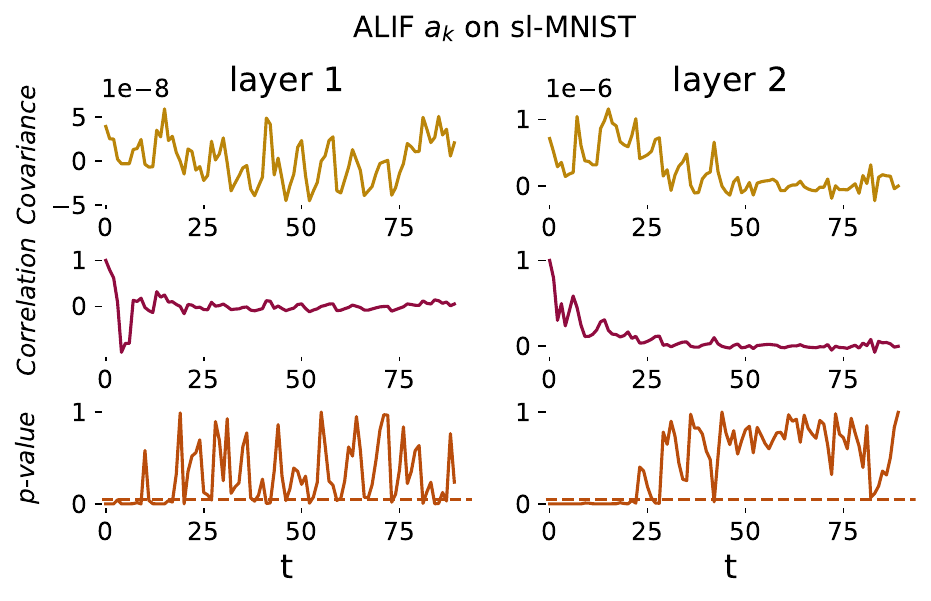}
    \includegraphics[width=.43\textwidth]{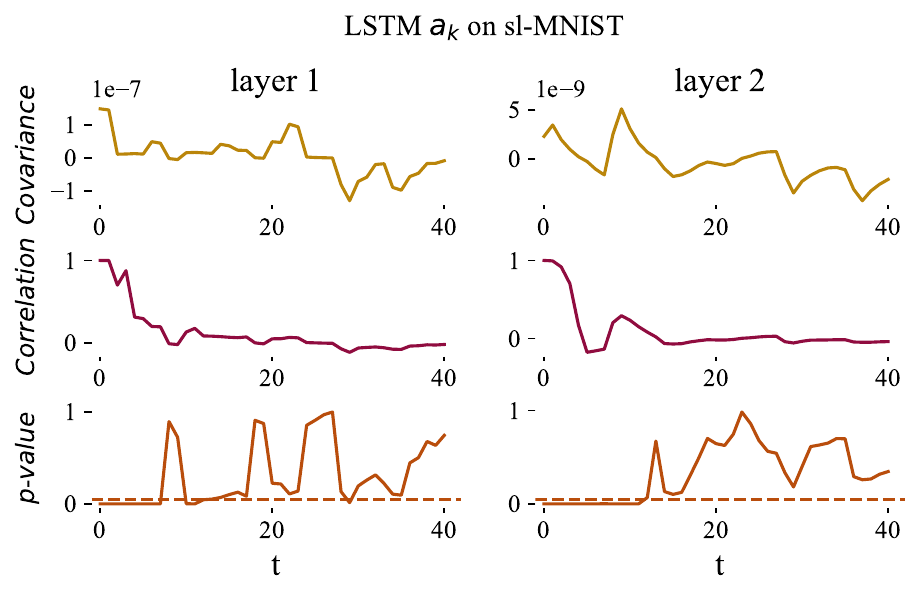}
    \includegraphics[width=.45\textwidth]{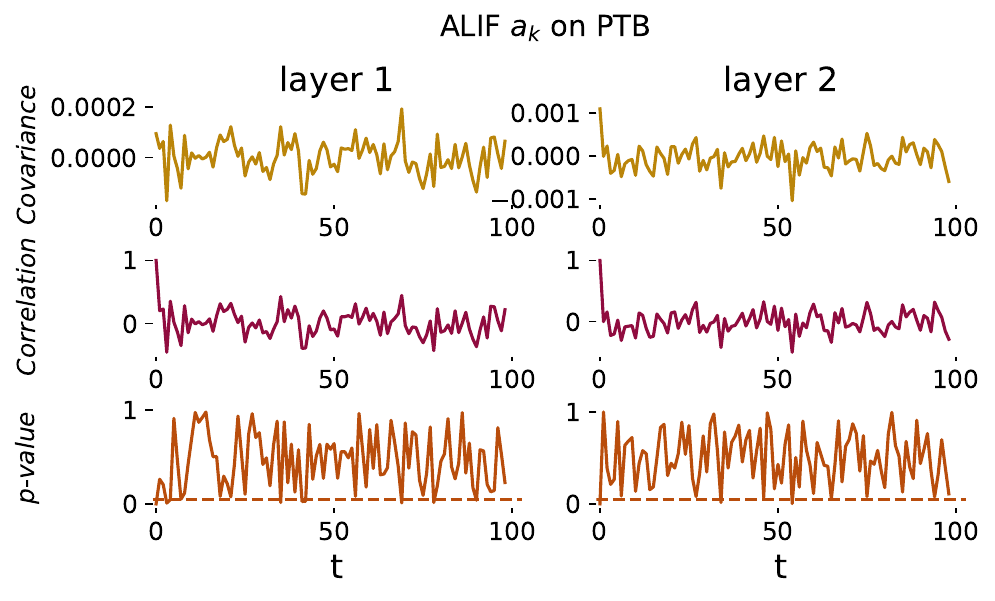}
    \includegraphics[width=.43\textwidth]{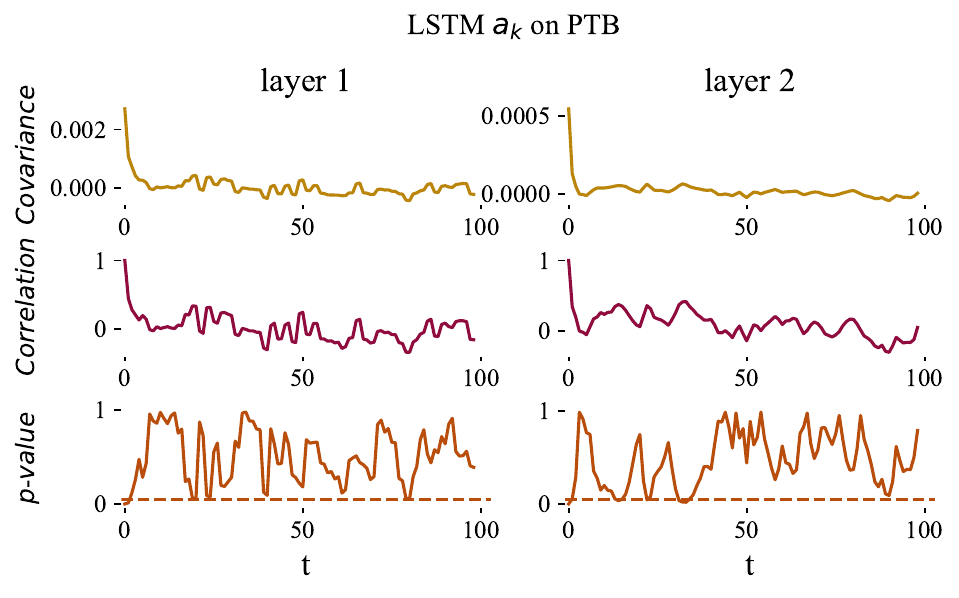}
    \caption{\textbf{Covariance of  $a_k$ decays with time or remains very close to zero after LSC pre-training.} Similar to the previous plot but this time, the networks were pre-trained to satisfy the LSC with the matrix norm induced by the vector $p$-norm with $p=2$, such that $\E a_k=1$. We can see that the covariances are very small. We can also see that the correlation quickly decays to zero, and the $p$-values above $0.05$ indicate that the small correlations are not significant. This shows that both ALIF and LSTM retain Decaying Covariance after pre-training, as defined in App.~\ref{app:lsc} and as necessary to prove Thm.~\ref{thm:lsc}.}
    \label{fig:dc2}
\end{figure}